\documentclass[10pt,letterpaper]{article}
\usepackage[top=1in,bottom=1in,left=1in,right=1in]{geometry}

\usepackage[utf8]{inputenc}
\usepackage{hyperref}
\usepackage{url}
\usepackage{booktabs}
\usepackage{amsfonts}
\usepackage{nicefrac}
\usepackage{xcolor}

\usepackage{amsmath, amssymb}
\usepackage{amsthm}
\usepackage{graphicx}
\usepackage{mathtools}

\usepackage{float}
\usepackage{makecell}
\usepackage{enumitem}

\newtheorem{theorem}{Theorem}

\newtheorem{lemma}{Lemma}

\usepackage{natbib}
\usepackage{authblk}

\title{Matrix Completion from General Deterministic Sampling Patterns}

\author[1]{Hanbyul Lee}
\author[2]{Rahul Mazumder}
\author[1]{Qifan Song}
\author[3]{Jean Honorio}
\affil[1]{Department of Statistics, Purdue University}
\affil[2]{Sloan School of Management, MIT}
\affil[3]{Department of Computer Science, Purdue University}

\date{\today}

\begin{document}

\maketitle

\begin{abstract}
  
Most of the existing works on provable guarantees for low-rank matrix completion algorithms rely on some unrealistic assumptions such that matrix entries are sampled randomly or the sampling pattern has a specific structure.
In this work, we establish theoretical guarantee for the exact and approximate low-rank matrix completion problems which can be applied to any deterministic sampling schemes.
For this, we introduce a graph having observed entries as its edge set, and investigate its graph properties involving the performance of the standard constrained nuclear norm minimization algorithm.
We theoretically and experimentally show that the algorithm can be successful as the observation graph is well-connected and has similar node degrees.
Our result can be viewed as an extension of the works by \cite{bhojanapalli2014universal} and \cite{burnwal2020deterministic}, in which the node degrees of the observation graph were assumed to be the same.
In particular, our theory significantly improves their results when the underlying matrix is symmetric.
  
\end{abstract}

\section{Introduction}

Low-rank matrix completion is to exactly or approximately recover an underlying rank-$r$ matrix from a small number of observed entries of the matrix.
It has received much attention in a wide range of applications including collaborative filtering \citep{goldberg1992using}, phase retrieval \citep{candes2015phase} and image processing \citep{chen2004recovering}.
In research on establishing provable guarantees for low-rank matrix completion methods, typical assumptions considered are as follows: first, the underlying matrix is incoherent; 
second, observable matrix entries are sampled according to a probabilistic (usually uniform) model.
However, the latter assumption is easily violated in numerous situations; it is unlikely realistic for the sampling patterns to be uniformly at random outside experimental settings, and it may not be even reasonable to model the sampling patterns as random.

With this motivation, we aim to tackle the low-rank matrix completion problem without imposing any model assumptions on the sampling patterns.
Even though there have been several works on non-random sampling schemes, 
they have imposed additional structural assumptions on the sampling pattern which are also not applicable to many real-world scenarios.
For example, \cite{heiman2014deterministic}, \cite{bhojanapalli2014universal} and \cite{burnwal2020deterministic} assumed that the number of observed entries is the same for each row and column of the matrix,
and \cite{bishop2014deterministic} introduced systematic assumptions on the subsets of the observed entries.

In this paper, we derive a provable guarantee for matrix completion that can be applied to any general deterministic sampling patterns.
Our approach is to consider an `observation graph' whose edge set is the observed entries of the underlying matrix, and investigate its graph properties which associate with the solvability of the matrix completion problem.
We analyze the standard constrained nuclear norm minimization method \citep{candes2009exact, candes2010matrix, recht2011simpler} for the exact and approximate matrix completion problems,
and find sufficient conditions to achieve success of the matrix completion algorithm.

Our study identifies some key graph properties that are simple, interpretable, and applicable to any bipartite or undirected graphs, making them suitable to describe any general deterministic sampling patterns.
They represent how connected the graph is and how similar the node degrees of the graph are to each other.
Through these graph properties, we theoretically and experimentally demonstrate that the nuclear norm minimization method is successful when the observation graph is well-connected and has similar node degrees.
This finding is logical because the more entries are observed, the more advantageous it is for matrix completion, and if entries are not sufficiently observed in any row or column, the missing entries in that row or column are difficult to recover.

Lastly, it is noteworthy that our study extends the works by \cite{bhojanapalli2014universal} and \cite{burnwal2020deterministic}, in which 
the observation graph was fixed to be regular, i.e., each node of the observation graph has the same degree.
More importantly, our theory significantly improves their results when the underlying matrix is symmetric. 
The results of \cite{bhojanapalli2014universal} and \cite{burnwal2020deterministic} have limitations that they rely on a strong matrix incoherence assumption and prove only a sub-optimal sample complexity rate.
When the underlying matrix is symmetric, our derivation only requires the standard incoherence assumption.
Furthermore, if we apply our theorem to the case that the observation graph is regular, we can show that matrix completion is achievable with a near-optimal sample complexity rate. 
This represents a significant advancement compared to the results demonstrated in \cite{bhojanapalli2014universal} and \cite{burnwal2020deterministic}.

\paragraph{Notation}
Matrices are bold capital (e.g., $\pmb{A}$), vectors are bold lowercase (e.g., $\pmb{a}$), and scalars or entries are not bold.
$A_{i,j}$ and $a_i$ represent the $(i,j)$-th and $i$-th entries of $\pmb{A}$ and $\pmb{a}$, respectively. 
$\pmb{A}_{i,:}$ represents the $i$-th row of $\pmb{A}$ (but in column format) and $\pmb{A}_{:,j}$ represents the $j$-th column of $\pmb{A}$.
For any positive integer $n$, we denote $[n]:=\{1,\dots,n\}$.
$\|\pmb{a}\|$ represents the $l_2$ norm of $\pmb{a}$.
$\|\pmb{A}\|$, $\|\pmb{A}\|_F$ and $\|\pmb{A}\|_*$ indicate the spectral, Frobenius and nuclear norms of $\pmb{A}$, respectively.
$\pmb{A}^\top$ is the transpose of $\pmb{A}$.
For any positive integer $n$, we denote by $\pmb{I}_n$ the $n$-dimensional identity matrix.
$f(x) = O(g(x))$ or $f(x) \lesssim g(x)$ means that there exists a positive constant $C$ such that $f(x) \leq C g(x)$ asymptotically.

\section{Problem Definition and Previous Results}

In this section, we introduce the problems of exact and approximate matrix completion, along with the constrained nuclear norm minimization method as a solution approach. 
We also review related works which were based on the assumption that sampling is uniformly at random or deterministic in a restrictive setting.
We will improve their results in Section \ref{sec:main}.

\subsection{Problem Definition}

Let $\pmb{M} \in \mathbb{R}^{n_1\times n_2}$ be an unknown rank-$r$ matrix 
and $\pmb{M} = \pmb{U}\pmb{\Sigma}\pmb{V}^\top$ be the singular value decomposition of $\pmb{M}$, where $\pmb{U}\in\mathbb{R}^{n_1\times r}$ and $\pmb{V}\in\mathbb{R}^{n_2\times r}$ are the left and right singular matrices, respectively.
Suppose that we only observe the entries of $\pmb{M}$ over a fixed sampling set $\Omega \subseteq [n_1]\otimes[n_2]$ without or with additive noise.
We define the sampling operator $P_{\Omega}:\mathbb{R}^{n_1\times n_2}\rightarrow \mathbb{R}^{n_1\times n_2}$ as follows:
$$
P_{\Omega}(\pmb{M}) = 
\begin{cases}
M_{i,j} & \text{if}~~ (i,j)\in\Omega,
\\
0 & \text{otherwise}.
\end{cases}
$$
In the scenario where observed entries are \emph{noiseless},
our goal is to recover $\pmb{M}$ exactly by using $P_{\Omega}(\pmb{M})$, and we call this the \textbf{exact matrix completion} problem.

In the scenario where observed entries are corrupted by additive noise,
we denote the noisy observation matrix by
$$
\pmb{Y} = P_{\Omega}(\pmb{M}+\pmb{E}),
$$
where $\pmb{E}\in \mathbb{R}^{n_1\times n_2}$ is a noise matrix.
We assume that $E_{i,j}$'s independently follow a sub-Gaussian distribution,
i.e., $\mathbb{E}e^{\theta E_{i,j}} \leq e^{\frac{\sigma^2 \theta^2}{2}}$ for any $\theta \geq 0$ and some $\sigma \geq 0$.
Our objective in this scenario is to obtain a precise estimation of $\pmb{M}$ by using $\pmb{Y}$, which we refer to as the \textbf{approximate matrix completion} problem.

To solve the exact and approximate matrix completion problems, we consider the standard constrained nuclear norm minimization algorithms.
For exact matrix completion, the following approach is used:
\begin{gather}
\underset{\pmb{X}\in\mathbb{R}^{n_1 \times n_2}}{\min} \| \pmb{X}\|_* 
~~
\text{subject to}~~ P_{\Omega}(\pmb{X}) = P_{\Omega}(\pmb{M}).
\label{eq:nnm}
\end{gather}
For approximate matrix completion, we consider the following modified approach:
\begin{gather}
\underset{\pmb{X}\in\mathbb{R}^{n_1 \times n_2}}{\min} \| \pmb{X}\|_* 
~~
\text{subject to}~~ \|P_{\Omega}(\pmb{X}-\pmb{Y})\|_F \leq \delta.
\label{eq:cnnm}
\end{gather}

\subsection{Previous Results on Nuclear Norm Minimization Method}

In this section, we introduce previous results on the constrained nuclear norm minimization method which are closely related to our work.
We provide additional discussion in Appendix \ref{appendix_sec:existing_works} regarding other studies that have addressed matrix completion under deterministic sampling schemes, though do not align with the specific goals and techniques of our research.

\subsubsection{When Sampling is Uniformly at Random}

\eqref{eq:nnm} and \eqref{eq:cnnm} have been thoroughly studied over the years, and most of the works have been based on the assumption that $\Omega$ is sampled uniformly at random.
For exact matrix completion, it is well-known that \eqref{eq:nnm} can succeed with an optimal sample complexity, which has been proven by \cite{recht2011simpler} as follows.

\begin{theorem}[Theorem 2 in \cite{recht2011simpler}]
\label{thm:recht}
Let $\pmb{M}$ be an $n_1\times n_2$ matrix of rank $r$ satisfying the following incoherence assumption:
\begin{itemize}
\item[\textup{\textbf{A1}}] For any $i\in[n_1]$ and $j\in[n_2]$, 
$\|\pmb{U}_{i,:}\|^2 \leq \frac{\mu_0 r}{n_1}$ and $\|\pmb{V}_{j,:}\|^2 \leq \frac{\mu_0 r}{n_2}$ for some positive $\mu_0$.
\end{itemize}
Suppose that $|\Omega|$ entries of $\pmb{M}$ are observed uniformly at random.
If the following condition is satisfied:
$$
|\Omega| \gtrsim \mu_0 r(n_1 + n_2)\log^2(\max(n_1, n_2)),
$$
then $\pmb{M}$ is the unique optimum of the problem \eqref{eq:nnm} with high probability.
\end{theorem}

For approximate matrix completion, \cite{candes2010matrix} derived the following error bound for the solution of \eqref{eq:cnnm}.

\begin{theorem}[Theorem 7 in \cite{candes2010matrix}]
\label{thm:plan}
Assume \textup{\textbf{A1}} as in Theorem \ref{thm:recht}.
Suppose that the entries of noisy matrix $\pmb{M} + \pmb{E}$ are observed uniformly at random with a fixed probability $p$,
and let $\Omega$ be the set of observed entries.
Suppose that $E_{i,j}$'s independently follow a sub-Gaussian distribution with parameter $\sigma$.
If the following condition is satisfied:
$$
|\Omega| \gtrsim \mu_0 r(n_1 + n_2)\log^2(\max(n_1, n_2)),
$$
then for $\delta \geq c\sigma\sqrt{|\Omega|}$, the solution $\hat{\pmb{M}}$ of the problem \eqref{eq:cnnm} obeys
$\|\pmb{M} - \hat{\pmb{M}} \|_F \leq 4\delta \sqrt{\frac{C \min(n_1, n_2)}{p}} + 2\delta $
with high probability, where $c$ and $C$ are some positive constants.
\end{theorem}

However, the assumption that sampling is uniformly at random is unrealistic, and the solvability of \eqref{eq:nnm} and \eqref{eq:cnnm} under deterministic sampling schemes has remained relatively unclear.

\subsubsection{When Sampling is Deterministic in a Restrictive Setting}

\cite{bhojanapalli2014universal} and \cite{burnwal2020deterministic}
presented the theoretical guarantees of \eqref{eq:nnm} and \eqref{eq:cnnm} when  $\Omega$ is deterministic.
However, their works are restricted to the scenario where the number of observed entries in each row and column is the same.
Furthermore, they have several additional limitations as well.
We first introduce their theorems on exact and approximate matrix completion below.

\begin{theorem}[Theorem 4.2 in \cite{bhojanapalli2014universal}]
\label{thm:bhojanapalli}
Let $\pmb{M}$ be an $n_1\times n_2$ matrix of rank $r$, and suppose that we observe its entries over a fixed sampling set $\Omega$.
Suppose that the number of observed entries in each row and column is the same, denoted as $d_1$ and $d_2$, resp.
Assume \textup{\textbf{A1}} as in Theorem \ref{thm:recht} and assume that
\begin{itemize}
\item[\textup{\textbf{A2}}]
For any $S\subseteq [n_1]$ s.t. $|S|=d_2$, 
$\| \frac{n_1}{d_2}\sum_{i\in S} \pmb{U}_{i,:} \pmb{U}_{i,:}^\top - \pmb{I}_{n_1}\| \leq \theta$
and
\\
for any $S\subseteq [n_2]$ s.t. $|S|=d_1$,
$\| \frac{n_2}{d_1}\sum_{j\in S} \pmb{V}_{j,:} \pmb{V}_{j,:}^\top - \pmb{I}_{n_2}\| \leq \theta$ 
for some positive $\theta$.
\end{itemize}
If $\theta \leq \frac{1}{6}$ and
$
|\Omega| \geq 36C^2 \mu_0^2 r^2 max(n_1, n_2)
$
for some positive constant $C$,
then $\pmb{M}$ is the unique optimum of the problem \eqref{eq:nnm}.
\end{theorem}

\begin{theorem}[Theorem 7 in \cite{burnwal2020deterministic}\footnotemark]
\label{thm:burnwal}
Let $\pmb{M}$ be an $n_1\times n_2$ matrix of rank $r$, and suppose that we observe the entries of noisy matrix $\pmb{M}+\pmb{E}$ over a fixed sampling set $\Omega$, where $E_{i,j}$'s independently follow a sub-Gaussian distribution with parameter $\sigma$.
Suppose that the number of observed entries in each row and column is the same, denoted as $d_1$ and $d_2$, resp.
Assume \textup{\textbf{A1}} and \textup{\textbf{A2}} as in Theorem \ref{thm:bhojanapalli}.
If the following condition is satisfied:
\begin{align*}
\phi + \sqrt{\frac{n_1 n_2 r (\theta^2 + \phi^2)}{|\Omega| (1-\theta -\phi)}} \leq \frac{1}{2}
\end{align*}
where $\phi = \frac{c\mu_0 r}{\sqrt{d}}$ for some positive constant $c$,
then for $\delta \geq 4\sigma\sqrt{|\Omega|} + 2\sigma\sqrt{\log(\eta^{-1})}$, the solution $\hat{\pmb{M}}$ of the problem \eqref{eq:cnnm} obeys
$\|\pmb{M} - \hat{\pmb{M}} \|_F \leq 4\delta \sqrt{\frac{C \min(n_1, n_2)}{p}} + 2\delta$
with probability at least $1-\eta$, where $p = \frac{|\Omega|}{n_1 n_2}$ and $C$ is some positive constant.

\end{theorem}

\footnotetext{Here, we slightly modified the expressions of the original paper to make comparison to our result easier. We note that the error bound and the rate of sample complexity remain unchanged.}

The limitations of Theorems \ref{thm:bhojanapalli} and \ref{thm:burnwal} are as follows:
\begin{itemize}[leftmargin=2em]
\item First and foremost, their works are restricted to the case where the number of observed entries in each row and column is the same.
\item Furthermore, they consider a stronger matrix incoherence assumption \textbf{A2} than the standard assumption \textbf{A1}. 
However, 
the necessity of such a stronger incoherence assumption is not clear.
\item The claimed rate of sample complexity $O(\max(n_1, n_2) r^2)$ for exact matrix completion in Theorem \ref{thm:bhojanapalli} in fact remains invalid since there was an error in their proof (see Appendix in \cite{burnwal2020deterministic}).
Hence, it is unclear whether this near-optimal rate is indeed achievable.
\item 
Theorem \ref{thm:burnwal} presents the error bound for approximate matrix completion with the same rate as in Theorem \ref{thm:plan}.
However, the rate of sample complexity to achieve this bound in Theorem \ref{thm:burnwal} is at least $O(\{\max(n_1, n_2) r\}^{1.5})$ or worse, which is much greater than that of Theorem \ref{thm:plan}.
\end{itemize}

Our main interest lies in whether we can prove a better sample complexity rate without relying on the stronger incoherence assumption \textbf{A2},
for any deterministic sampling schemes.
In the next section, we will show that this is achievable for symmetric matrices.

\section{Matrix Completion under General Deterministic Sampling Schemes}
\label{sec:main}

\subsection{Observation Graph Properties}
\label{subsec:graph_terminologies}

Before presenting our main theorems, we first introduce our tools to handle general deterministic sampling schemes: the observation graph and its graph properties.

Given the observed entries of a matrix over a fixed sampling set $\Omega$, we can construct a graph with these entries as its edge set. We refer to this graph as the observation graph.
We denote it by $\mathcal{G} = (\mathcal{U},\mathcal{V},\mathcal{E})$, which is a bipartite graph where $\mathcal{U} = [n_1]$, $\mathcal{V} = [n_2]$, and $(i,j)\in\mathcal{E}$ if and only if $(i,j)\in\Omega$.
$\pmb{A}_{\mathcal{G}} \in \mathbb{R}^{n_1 \times n_2}$ indicates the biadjacency matrix corresponding to the graph $\mathcal{G}$.

Below are several common graph terminologies which will be useful in our statements.
\begin{itemize}[leftmargin=2em]
\item $\bar{\mathcal{G}}$:
For a graph $\mathcal{G}$, $\bar{\mathcal{G}}$ denotes its \emph{complement graph}, i.e., 
$\bar{\mathcal{G}}$ has the same vertex sets as $\mathcal{G}$ but its edge set is the complement of that of $\mathcal{G}$.
\item $\Delta_{i,\mathcal{U},\mathcal{G}}$, $\Delta_{j,\mathcal{V},\mathcal{G}}$: 
For a graph $\mathcal{G} = (\mathcal{U},\mathcal{V},\mathcal{E})$,
$\Delta_{i,\mathcal{U},\mathcal{G}}$ and $\Delta_{j,\mathcal{V},\mathcal{G}}$ indicate the \emph{node degrees} of the $i$-th vertex in the set $\mathcal{U}$ and the $j$-th vertex in the set $\mathcal{V}$, respectively.
\item $\Delta_{\max, \mathcal{G}}$: 
For a graph $\mathcal{G} = (\mathcal{U},\mathcal{V},\mathcal{E})$,
we define $\Delta_{\max, \mathcal{G}}$ as
$\frac{\max_{i\in[n_1]} \Delta_{i,\mathcal{U},\mathcal{G}} + \max_{j\in[n_2]} \Delta_{j,\mathcal{V},\mathcal{G}}}{2}$, that is, it indicates the \emph{average of the maximum node degrees} of two vertex sets $\mathcal{U}$ and $\mathcal{V}$. 
\item $\varphi_{\mathcal{G}}$: 
For a graph $\mathcal{G}$,
we denote by $\varphi_{\mathcal{G}}$ the \emph{algebraic connectivity} of a graph $\mathcal{G}$, i.e., the second-smallest eigenvalue of the Laplacian matrix of $\mathcal{G}$.
For a bipartite graph $\mathcal{G} = (\mathcal{U},\mathcal{V},\mathcal{E})$, the Laplacian matrix is defined as $(n_1 + n_2) \times (n_1 + n_2)$-dimensional matrix
$
\begin{psmallmatrix}
\pmb{D}_{\mathcal{U},\mathcal{G}} & \pmb{0} \\
\pmb{0} & \pmb{D}_{\mathcal{V},\mathcal{G}}
\end{psmallmatrix}
-
\begin{psmallmatrix}
\pmb{0} & \pmb{A}_{\mathcal{G}} \\ 
\pmb{A}_{\mathcal{G}}^\top & \pmb{0}
\end{psmallmatrix},
$
where 
$\pmb{D}_{\mathcal{U},\mathcal{G}}$ and $\pmb{D}_{\mathcal{V},\mathcal{G}}$ are diagonal matrices whose diagonal elements are the node degrees of the vertex sets $\mathcal{U}$ and $\mathcal{V}$, respectively.
\end{itemize}

Now, we define two kinds of important graph properties which involve our main theorems.
First, we introduce $\xi_{1,\mathcal{G}}$ and $\xi_{2,\mathcal{G}}$ indicating the \emph{deviations of node degrees}, defined as:
{\small
$$
\xi_{1,\mathcal{G}} := \sqrt{\frac{1}{n_2}\sum_{i\in[n_1]}\bigg(\Delta_{i,\mathcal{U},\mathcal{G}} - \frac{1}{n_1}\sum_{i\in[n_1]} \Delta_{i,\mathcal{U},\mathcal{G}} \bigg)^2},
~~
\xi_{2,\mathcal{G}} := \sqrt{\frac{1}{n_1}\sum_{j\in[n_2]}\bigg(\Delta_{j,\mathcal{V},\mathcal{G}} - \frac{1}{n_2}\sum_{j\in[n_2]} \Delta_{j,\mathcal{V},\mathcal{G}} \bigg)^2}.
$$
}
As the node degrees become more similar to each other, the values of $\xi_{1,\mathcal{G}}$ and $\xi_{2,\mathcal{G}}$ decrease, and when all node degrees have the same value, these values become 0.

Next, for a graph $\mathcal{G}$ such that $\Delta_{\max,\mathcal{G}} \geq \varphi_{\mathcal{G}}$ and $\Delta_{\max,\bar{\mathcal{G}}} \geq \varphi_{\bar{\mathcal{G}}}$, we define $\psi_{\mathcal{G}}$ as follows:
$$
\psi_{\mathcal{G}} := \max\{ \Delta_{\max,\mathcal{G}}-\varphi_{\mathcal{G}},~\Delta_{\max,\bar{\mathcal{G}}}-\varphi_{\bar{\mathcal{G}}} \},
$$
which represents \emph{how disconnected the graph is} and \emph{how much variation exists in node degrees}.
As the connectivity values of $\mathcal{G}$ and $\bar{\mathcal{G}}$ increase, and the maximum node degrees of $\mathcal{G}$ and $\bar{\mathcal{G}}$ are not significantly larger than the connectivity values, the value of $\psi_{\mathcal{G}}$ becomes smaller.

By using these graph properties, we will demonstrate that as 
the observation graph is more well-connected and has more even node degrees, we are able to solve the matrix completion problem more effectively.
We highlight that the aforementioned graph properties can be applied to \emph{any graphs}, thus they can be used to describe \emph{any deterministic sampling patterns}.

\subsection{Theoretical Results for Symmetric Matrices}

Now, we introduce our main theorems, which show the sufficient conditions for the matrix completion algorithms \eqref{eq:nnm} and \eqref{eq:cnnm} to be successful under a general deterministic sampling pattern.
We first focus on the case that the underlying matrix is symmetric.
In this case, our results significantly improve the previous theorems of \cite{bhojanapalli2014universal} and \cite{burnwal2020deterministic}.

We let $n = n_1 = n_2$ in this section.
We also note that the observation graph is an undirected graph containing loops here, and we write it as $\mathcal{G} = (\mathcal{U}, \mathcal{E})$.
Since $\xi_{1,\mathcal{G}}$ and $\xi_{2,\mathcal{G}}$ are the same, we write $\xi_{\mathcal{G}} = \xi_{1,\mathcal{G}}=\xi_{2,\mathcal{G}}$.

Below is the theorem of solvability of exact matrix completion for noiseless symmetric matrices. We defer the proof to Appendix \ref{appendix_subsec:proof_of_thm1}.

\begin{theorem}[Exact completion of symmetric matrix]
\label{thm:noiseless_symmetric}
Let $\pmb{M}$ be an $n\times n$ symmetric matrix of rank $r$ satisfying the following incoherence assumption:
\begin{itemize}
\item[\textup{\textbf{A1}}] For any $i\in[n]$, 
$\|\pmb{U}_{i,:}\|^2 \leq \frac{\mu_0 r}{n}$ for some positive $\mu_0$.
\end{itemize}
Suppose that we observe the entries of $\pmb{M}$ over a fixed sampling set $\Omega$, which is given by a graph $\mathcal{G}=(\mathcal{U},\mathcal{E})$ with the graph properties $\xi_{\mathcal{G}}$ and $\psi_{\mathcal{G}}$.
If the following condition is satisfied:
\begin{align}
\label{eq:noiseless_symmetric_condition}
|\Omega| >3 \mu_0 n r\cdot (2 \xi_{\mathcal{G}}+\psi_{\mathcal{G}}),
\end{align}
then $\pmb{M}$ is the unique optimum of the problem \eqref{eq:nnm}.
\end{theorem}

A notable aspect of the above theorem is that, unlike Theorem \ref{thm:bhojanapalli}, we do not require a strong matrix incoherence assumption \textbf{A2}.
Furthermore, even without the strong incoherence assumption, we can derive a near-optimal sample complexity result as follows:
if the observation graph is $d$-regular, i.e., the node degrees of the graph are all equal to $d$, and its adjacency matrix has the second largest singular value of $C\sqrt{d}$ for some positive constant $C$,
then the values of $|\Omega|$ and $2\xi_{\mathcal{G}} + \psi_{\mathcal{G}}$ can be replaced by $nd$ and $C\sqrt{d}$, respectively.
Accordingly, the sufficient condition \eqref{eq:noiseless_symmetric_condition} can be written as:
$$
nd > 3C\mu_0 nr\sqrt{d} ~~\Leftrightarrow~~
|\Omega| > 9C^2\mu_0^2 n r^2 = O(nr^2)
$$
in this case. This is a near-optimal rate of sample complexity when the rank $r$ is low enough.

For the approximate matrix completion problem of noisy symmetric matrices,
we can show the advanced rate of theoretical guarantee as well.
The following theorem shows the result, whose proof is given in Appendix \ref{appendix_subsec:proof_of_thm2}.

\begin{theorem}[Approximate completion of symmetric matrix]
\label{thm:noisy_symmetric}
Let $\pmb{M}$ be an $n\times n$ symmetric matrix of rank $r$ satisfying the assumption \textup{\textbf{A1}} as in Theorem \ref{thm:noiseless_symmetric}.
Suppose that we observe the entries of noisy matrix $\pmb{M}+\pmb{E}$ over a fixed sampling set $\Omega$, which is given by a graph $\mathcal{G}=(\mathcal{U},\mathcal{E})$ with the graph properties $\xi_{\mathcal{G}}$ and $\psi_{\mathcal{G}}$.
Also, we assume that $E_{i,j}$'s independently follow a sub-Gaussian distribution with parameter $\sigma$.
If the following condition is satisfied:
\begin{align*}
|\Omega| \gtrsim \mu_0 n r^{1.5} \cdot (2\xi_{\mathcal{G}} + \psi_{\mathcal{G}}),
\end{align*}
then
for $\delta \geq 4\sigma\sqrt{|\Omega|} + 2\sigma\sqrt{\log(\eta^{-1})}$,
the solution $\hat{\pmb{M}}$ of the problem \eqref{eq:cnnm} obeys
$\|\pmb{M} - \hat{\pmb{M}} \|_F \leq 4\delta \sqrt{\frac{C n}{p}} + 2\delta$
with probability at least $1-\eta$, where $p = \frac{|\Omega|}{n^2}$ and $C$ is some positive constant.
\end{theorem}

As in Theorem \ref{thm:noiseless_symmetric}, we do not assume the strong incoherence assumption \textbf{A2}.
Even without this assumption, we derive the error bound with the same rate as in Theorem \ref{thm:plan}.
In particular, when the observation graph is $d$-regular and its adjacency matrix has the second largest singular value of $C\sqrt{d}$, 
the rate of sample complexity to achieve this error bound becomes $O(nr^3)$, which is comparable to that of Theorem \ref{thm:plan} and  significantly better than that of Theorem \ref{thm:burnwal}
when the rank $r$ is low enough.
Therefore, our finding represents a meaningful improvement.

\subsection{Theoretical Results for Rectangular Matrices}

Unfortunately, it is not trivial to derive comparable theoretical guarantee for general rectangular matrices to that for symmetric matrices.
However, in Theorems \ref{thm:noiseless_rectangular} and \ref{thm:noisy_rectangular} below, we extend the result of \cite{burnwal2020deterministic} to the case that the observation pattern is deterministic and general, for the exact and approximate completion problems, respectively, by utilizing the graph properties introduced in Section \ref{subsec:graph_terminologies}.
We defer the proofs to Appendix \ref{appendix_sec:proof_rectangular}.

\begin{theorem}[Exact completion of rectangular matrix]
\label{thm:noiseless_rectangular}
Let $\pmb{M}$ be an $n_1\times n_2$ matrix of rank $r$,
and suppose that we observe its entries over a fixed sampling set $\Omega$, which is given by a graph $\mathcal{G}=(\mathcal{U},\mathcal{V},\mathcal{E})$ with the graph properties $\xi_{1, \mathcal{G}}$, $\xi_{2, \mathcal{G}}$ and $\psi_{\mathcal{G}}$.
Assume that $\pmb{M}$ satisfies the followings:
\begin{itemize}[leftmargin=2em]
\item[\textup{\textbf{A1}}] For any $i\in[n_1]$ and $j\in[n_2]$, 
$\|\pmb{U}_{i,:}\|^2 \leq \frac{\mu_0 r}{n_1}$ and $\|\pmb{V}_{j,:}\|^2 \leq \frac{\mu_0 r}{n_2}$ for some positive $\mu_0$.
\item[\textup{\textbf{A2}}]
For $\forall S\subseteq [n_1]$ s.t. $\min_{j}\Delta_{j,\mathcal{V},\mathcal{G}} \leq |S| \leq \max_{j}\Delta_{j,\mathcal{V},\mathcal{G}}$, 
$\| \frac{n_1 n_2}{|\Omega|}\sum_{i\in S} \pmb{U}_{i,:} \pmb{U}_{i,:}^\top - \pmb{I}_{n_1}\| \leq \theta$
and
\\
for $\forall S\subseteq [n_2]$ s.t. $\min_{i}\Delta_{i,\mathcal{U},\mathcal{G}} \leq |S| \leq \max_{i}\Delta_{i,\mathcal{U},\mathcal{G}}$, 
$\| \frac{n_1 n_2}{|\Omega|}\sum_{j\in S} \pmb{V}_{j,:} \pmb{V}_{j,:}^\top - \pmb{I}_{n_2}\| \leq \theta$
for some positive $\theta$.
\end{itemize}
If the following condition is satisfied:
\begin{align*}
\gamma + \sqrt{\frac{n_1 n_2 r (\theta^2 + \gamma^2)}{|\Omega| (1-\theta -\gamma)}} < 1 
\end{align*}
where 
$\gamma := \frac{\sqrt{n_1 n_2}}{|\Omega|}\cdot \mu_0 r \cdot(\xi_{1, \mathcal{G}}+\xi_{2, \mathcal{G}}+\psi_{\mathcal{G}})$,
then $\pmb{M}$ is the unique optimum of the problem \eqref{eq:nnm}.
\end{theorem}

\begin{theorem}[Approximate completion of rectangular matrix]
\label{thm:noisy_rectangular}
Let $\pmb{M}$ be an $n_1\times n_2$ matrix of rank $r$,
and suppose that we observe the entries of noisy matrix $\pmb{M}+\pmb{E}$ over a fixed sampling set $\Omega$, which is given by a graph $\mathcal{G}=(\mathcal{U},\mathcal{V},\mathcal{E})$ with the graph properties $\xi_{1, \mathcal{G}}$, $\xi_{2, \mathcal{G}}$ and $\psi_{\mathcal{G}}$.
Also, we assume that $E_{i,j}$'s independently follow a sub-Gaussian distribution with parameter $\sigma$.
Assume that $\pmb{M}$ satisfies \textup{\textbf{A1}} and \textup{\textbf{A2}} as in Theorem \ref{thm:noiseless_rectangular}.
If the following condition is satisfied:
\begin{align*}
\gamma + \sqrt{\frac{n_1 n_2 r (\theta^2 + \gamma^2)}{|\Omega| (1-\theta -\gamma)}} \leq \frac{1}{2},
\end{align*}
then 
for $\delta \geq 4\sigma\sqrt{|\Omega|} + 2\sigma\sqrt{\log(\eta^{-1})}$,
the solution $\hat{\pmb{M}}$ of the problem \eqref{eq:cnnm} obeys
$\|\pmb{M} - \hat{\pmb{M}} \|_F \leq 4\delta \sqrt{\frac{C \min(n_1, n_2)}{p}} + 2\delta$
with probability at least $1-\eta$, where $p = \frac{|\Omega|}{n_1 n_2}$ and $C$ is some positive constant.
\end{theorem}

As in Theorems \ref{thm:noiseless_symmetric} and \ref{thm:noisy_symmetric},
the sufficient conditions involve the graph properties $\xi_{1, \mathcal{G}}$, $\xi_{2, \mathcal{G}}$ and $\psi_{\mathcal{G}}$ which are applicable to any graphs, that is, they can address any observation patterns.
In the case where the observation graph is $d$-regular, the parameter $\phi$ in Theorem \ref{thm:burnwal} and $\gamma$ in the above theorems coincide, meaning that our result generalizes the result of \cite{burnwal2020deterministic}.
However, Theorems \ref{thm:noiseless_rectangular} and \ref{thm:noisy_rectangular} still suffer from the limitations of Theorem \ref{thm:burnwal}, namely they depend on the strong incoherence assumption and hold the sub-optimal sample complexity rate.
It remains an open question whether the results can be improved as in the case of symmetric matrices.

\section{Experimental Results}

\begin{figure}[t]
	\centering
	\includegraphics[width=1\textwidth]{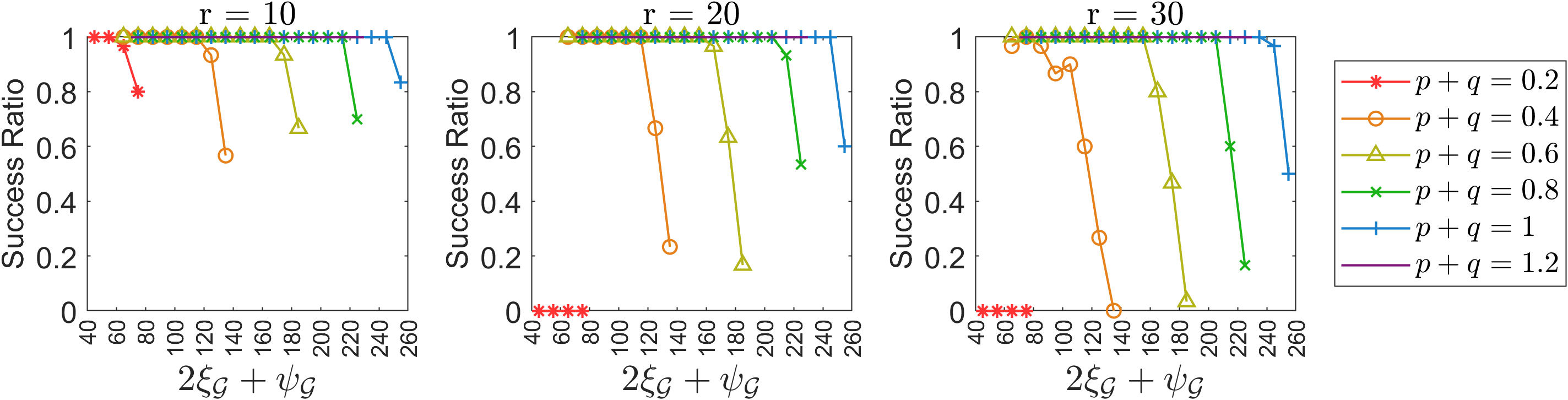}
	\caption{Success ratio of exact matrix completion versus graph property $2\xi_{\mathcal{G}} + \psi_{\mathcal{G}}$ for different rank $r$. Different line colors or markers indicate different values of observation probability $p+q$.}
	\label{fig:noiseless_different_rank}
\end{figure}

\subsection{Simulations}

The purpose of the simulation study is to illustrate the effects of the factors (e.g., graph properties such as $\xi_{\mathcal{G}}$ and $\psi_{\mathcal{G}}$) 
shown in our theorem, validating that they have an impact on the success of the matrix completion algorithm.
Due to space limitations, we only present the results for symmetric matrices, and defer the results for rectangular matrices to Appendix \ref{appendix_sec:experimental_results}.

We create synthetic data matrix as follows. 
We first generate the singular matrix $\pmb{U}\in \mathbb{R}^{500\times r}$ using standard normal distribution.
We then generate the rank-$r$ symmetric matrix $\pmb{M} \in \mathbb{R}^{500\times 500}$ using $\pmb{M} = \pmb{U}\pmb{U}^\top$.
In the scenario of noisy matrices, we randomly generate the entry-wise noise from a normal distribution with mean $0$ and standard deviation $\sigma$.
We try different values of rank $r \in \{10, 20, 30\}$ and noise parameter $\sigma \in \{10^{-4},10^{-5},10^{-6}\}$ in the experiments.

To generate observation graphs with various values of graph properties, we employ the stochastic block model. 
We first divide the nodes into two clusters and sample inter-cluster edges with a probability of $p \in (0,1)$ and intra-cluster edges with a probability of $q \in (0,1)$. 
For each $p+q \in \{0.2, 0.4, \dots, 1.2\}$, we try different values of $p$ and $q$ so that the graphs have diverse values of $2\xi_{\mathcal{G}} + \psi_{\mathcal{G}}$, the quantity influencing the solvability of matrix completion according to our theorem.
Specifically, we have $2\xi_{\mathcal{G}} + \psi_{\mathcal{G}}$ fall within one of the ranges $40$ to $50$, $50$ to $60$, $\dots$, or $250$ to $260$.

We use an Augmented Lagrangian Method \citep{lin2010augmented} to solve the constrained nuclear norm minimization problems \eqref{eq:nnm} and \eqref{eq:cnnm}.
When solving \eqref{eq:cnnm} for approximate matrix completion, 
we set the tuning parameter $\delta$ to be $4\sigma\sqrt{|\Omega|}$ as proven in Theorem \ref{thm:noisy_symmetric}.
For evaluation, we calculate the relative error $\frac{\|\pmb{M} - \hat{\pmb{M}} \|_F}{\| \pmb{M} \|_F}$ in each experiment.
In exact matrix completion, we consider a trial to be successful if the relative error is less than $0.01$, and compute the success ratio over $30$ trials with different random seeds.
In approximate matrix completion, we calculate the average relative error over $30$ trials.

Figure \ref{fig:noiseless_different_rank} shows the result of exact matrix completion in noiseless matrix case.
We can observe that as $2\xi_{\mathcal{G}} + \psi_{\mathcal{G}}$ or $r$ increases, or $p+q$ decreases (i.e., the number of observed entries decreases), the success ratio decreases, which supports Theorem \ref{thm:noiseless_symmetric}.
Figure \ref{fig:noisy_different_rank_sigma} demonstrates the result of approximate matrix completion in noisy matrix case.
In the three plots above, we can observe that as $2\xi_{\mathcal{G}} + \psi_{\mathcal{G}}$ or $r$ increases, or $p+q$ decreases, the average of relative errors increases.
In the three plots below, we can see that as the noise parameter $\sigma$ decreases, the average of relative errors decreases.
These observations are consistent with our findings in Theorem \ref{thm:noisy_symmetric}.

\begin{figure}[t]
	\centering
	\includegraphics[width=1\textwidth]{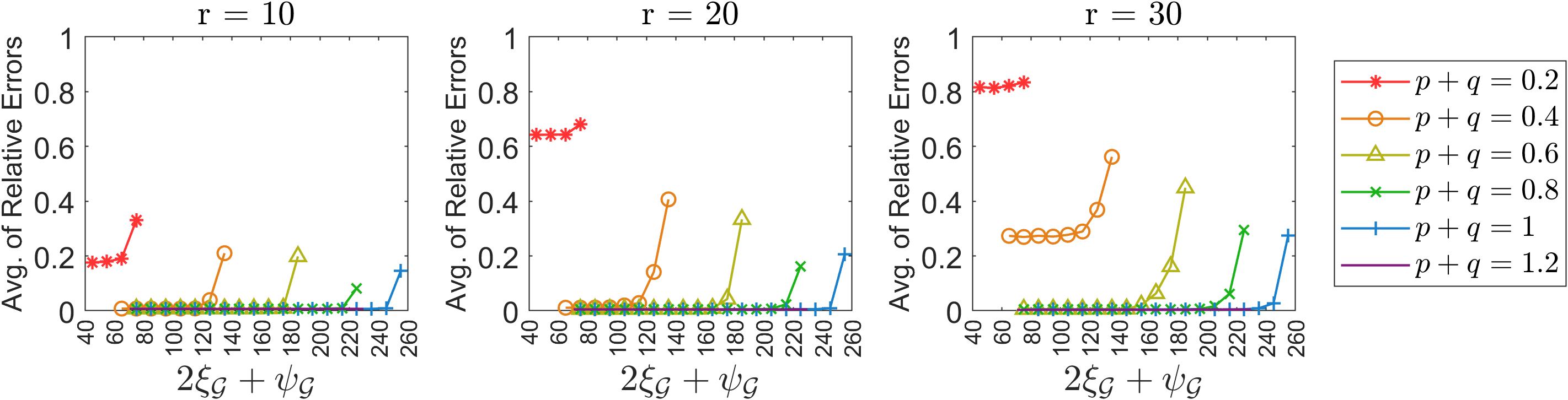}
	\\[1em]
	\includegraphics[width=1\textwidth]{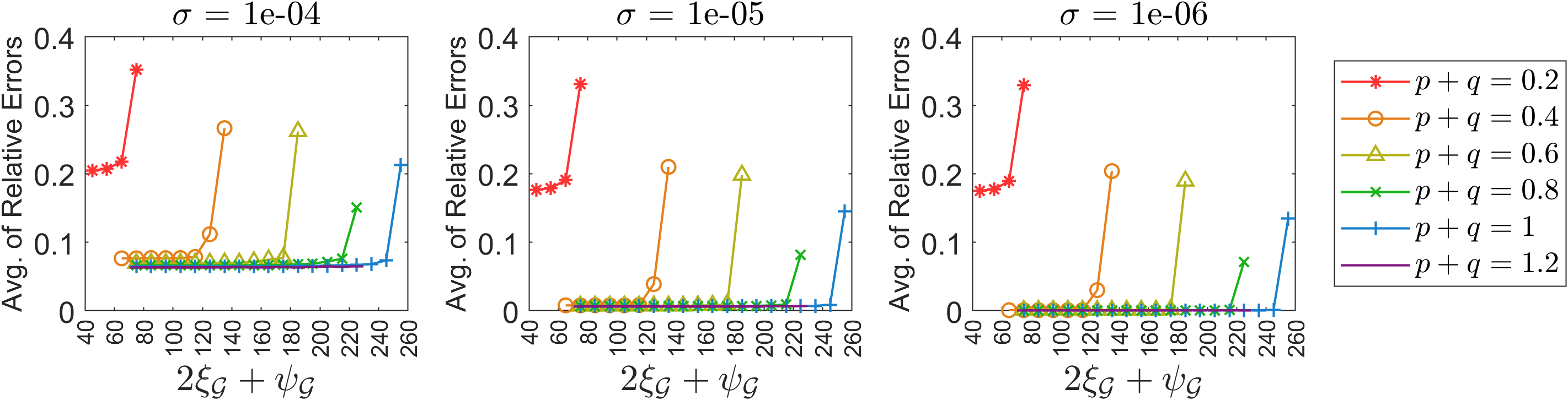}
	\caption{Average of relative errors in approximate matrix completion versus graph property $2\xi_{\mathcal{G}} + \psi_{\mathcal{G}}$.
	Three plots above show results for different rank $r$ with fixed noise parameter $\sigma=10^{-5}$.
	Three plots below are of different noise parameter $\sigma$ with fixed rank $r=10$.
Different line colors or markers indicate different values of observation probability $p+q$.}
	\label{fig:noisy_different_rank_sigma}
\end{figure}

Lastly, we want to verify whether the performance of the algorithm is solely determined by the factors derived in our theorems. 
Here, we focus on the case of noiseless symmetric matrices, while the results of other cases are deferred to Appendix \ref{appendix_sec:experimental_results}.
Our strategy is to utilize the rescaled parameter
$\frac{\textup{LHS of \eqref{eq:noiseless_symmetric_condition}}}{\textup{RHS of \eqref{eq:noiseless_symmetric_condition} without constant}} = \frac{|\Omega|}{\mu_0 r (2\xi_{\mathcal{G}} + \psi_{\mathcal{G}})}$.
If the pattern of the success ratio versus this rescaled parameter is the same across different settings, then we can empirically justify that the performance is solely determined by the factors in the rescaled parameter. 
This kind of approach has been used in \cite{wainwright2009sharp} for sparse linear regression.

In Figure \ref{fig:noiseless_overlap}, we use the same data set as in Figure \ref{fig:noiseless_different_rank} but calculate the rescaled parameter for each setting and plot the success ratio against the rescaled parameter.
We can observe that the curves share almost the same pattern across different settings of rank $r$. 
This empirical finding justifies the necessity and tightness of condition \eqref{eq:noiseless_symmetric_condition} in Theorem \ref{thm:noiseless_symmetric}.

\begin{figure}[t]
	\centering
	\includegraphics[width=0.3\textwidth]{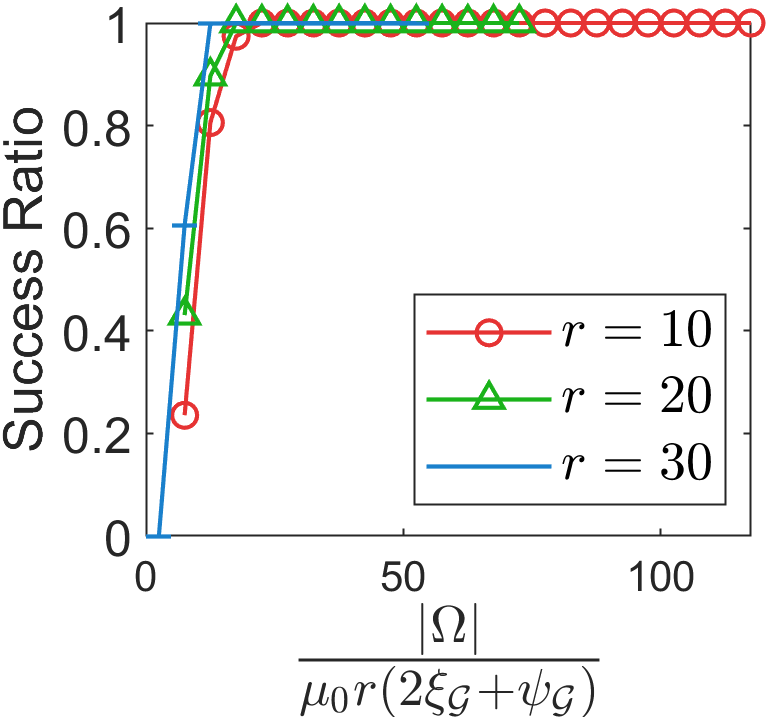}
	\caption{Success ratio of exact matrix completion versus rescaled parameter for different rank $r$.}
	\label{fig:noiseless_overlap}
\end{figure}

\subsection{Real Data Analysis}

The goal of the real data analysis is to demonstrate that the observation patterns of the actual data sets deviate from the uniform random sampling, which will support the rationale of our research.
We utilize the graph properties we introduced in our theorems for the comparison.

We consider the following common benchmark data sets for collaborative filtering: MovieLens (100K and 1M) \citep{maxwell2015movielens}, Flixster, and Douban.
MovieLens 100K (ML 100K) and 1M (ML 1M) data sets consist of 100,000 ratings from 943 users on 1682 movies, and 1,000,209 ratings from 6040 users on 3952 movies, respectively.
For Flixster and Douban, we use preprocessed subsets of these data sets provided by \cite{monti2017geometric}.
Flixster and Douban data sets consist of 26,173 and 136,891 ratings from 3000 users on 3000 items, respectively.

For comparison with uniform random sampling, we first generate graphs from the Erd\H{o}s-R\'enyi model with the same dimension as the observation graph of each data set, and pick 30 graphs with similar density to the real observation graph.
We then calculate the averages of the graph properties $\xi_{1,\mathcal{G}}$, $\xi_{2,\mathcal{G}}$ and $\psi_{\mathcal{G}}$ for these 30 graphs and compare them to those of the real observation graph.

Table \ref{table:real_data_graph_properties} summarizes the graph property values.
We can see a significant difference between the values of the real observation graph and the synthetic random graphs in each data set.
In particular, we find that the real observation graph has larger values of the graph properties than those of the random graphs in each case,
indicating that it is more difficult to satisfy the sufficient condition for matrix completion according to our theorem for the real observation.

To examine whether this is true,
we conduct experiment using synthetic rank-$1$ noiseless matrices with the same dimension as each data set.
We apply the real observation graph and the random graphs to the synthetic matrices to create incomplete matrices,
and then run the algorithm \eqref{eq:nnm}.
We repeat 30 trials with different random seeds, and calculate the average of relative errors.

\begin{table}[t]
\caption{Graph properties of real observation graphs of four benchmark data sets and averages of graph properties of corresponding random graphs generated from Erd\H{o}s-R\'enyi model.}
\label{table:real_data_graph_properties}
\centering
\begin{tabular}{lcccccccc}
\toprule
& \multicolumn{4}{c}{Real Observation Graph} & \multicolumn{4}{c}{Random Graphs from Erd\H{o}s-R\'enyi Model} \\ \cmidrule(lr){2-5} \cmidrule(lr){6-9}
Dataset & Density & $\psi_{\mathcal{G}}$ & $\xi_{1,\mathcal{G}}$ & $\xi_{2,\mathcal{G}}$ & \makecell{Avg. of \\ Density} & \makecell{Avg. of \\ $\psi_{\mathcal{G}}$} & \makecell{Avg. of \\ $\xi_{1,\mathcal{G}}$} & \makecell{Avg. of \\ $\xi_{2,\mathcal{G}}$} \\ \midrule
ML 100K  & 0.0630  & 942.08  & 75.53  & 107.32 & 0.0630 & 401.57 & 7.51 & 10.02   \\
ML 1M   & 0.0419  & 3348.11 & 238.27 & 305.38  & 0.0419 & 1098.11 & 15.57 & 12.59     \\
Flixster  & 0.0029  & 173.94  & 13.79  & 7.09  & 0.0029 & 22.70 & 2.95 & 2.95   \\
Douban  & 0.0152  & 120.89  & 20.10  & 20.05   & 0.0152 & 51.39 & 6.73 & 6.70  \\ 
\bottomrule 
\end{tabular}
\end{table}

\begin{table}[t]
\caption{Average of relative errors in matrix completion for synthetic data set
corresponding to each of four benchmark data sets, where real observation graph or random graph is applied.}
\label{table:real_data_error}
\centering
\begin{tabular}{lcc}
\toprule
         & \multicolumn{2}{c}{Average of Relative Errors}                                  \\ \cmidrule(lr){2-3} 
Dataset  & \multicolumn{1}{c}{Real Observation Graph} & Random Graphs from Erd\H{o}s-R\'enyi Model \\ \midrule
ML 100K  & \multicolumn{1}{c}{0.2335} & 0.000025  \\
ML 1M    & \multicolumn{1}{c}{0.3050} & 0.000032  \\
Flixster & \multicolumn{1}{c}{0.5582} & 0.062454  \\
Douban   & \multicolumn{1}{c}{0.0174} & 0.000018 \\ 
\bottomrule 
\end{tabular}
\end{table}

Table \ref{table:real_data_error} demonstrates the experimental results.
We can check that the relative error is much larger when the real observation graph is applied.
This result implies that it is more difficult to achieve successful matrix completion when the graph properties of $\xi_{1,\mathcal{G}}$, $\xi_{2,\mathcal{G}}$ and $\psi_{\mathcal{G}}$ have larger values,
which supports our theorem.
Furthermore, it shows that matrix completion in a real-world scenario is indeed more challenging than what is expected under a uniform random sampling scheme.
This highlights the importance of research on general sampling schemes, such as the one in our paper.

\section{Concluding Remarks}

In this paper, we establish the provable guarantees of exact and approximate matrix completion algorithms based on nuclear norm minimization, 
without any probabilistic or structural assumptions on the sampling schemes.
We utilize the observation graph and its properties to address a general non-random sampling scheme.
By using the graph properties, we theoretically and experimentally demonstrate that the nuclear norm minimization method is successful when the observation graph is well-connected and has similar node degrees.
It is notable that for symmetric matrices, our theorem significantly improves the existing works, which is supported by empirical evidence.
It remains an open question whether our result on symmetric matrices can be extended for general rectangular matrices.

\bibliography{ref}
\bibliographystyle{plainnat}

\newpage

\appendix

\section{Notation}

Matrices are bold capital (e.g., $\pmb{A}$), vectors are bold lowercase (e.g., $\pmb{a}$), and scalars or entries are not bold.
$A_{i,j}$ and $a_i$ represent the $(i,j)$-th and $i$-th entries of $\pmb{A}$ and $\pmb{a}$, respectively. 
$\pmb{A}_{i,:}$ represents the $i$-th row of $\pmb{A}$ (but in column format) and $\pmb{A}_{:,j}$ represents the $j$-th column of $\pmb{A}$.
For any positive integer $n$, we denote $[n]:=\{1,\dots,n\}$.
$\|\pmb{a}\|$ represents the $l_2$ norm of $\pmb{a}$.
$\|\pmb{A}\|$, $\|\pmb{A}\|_F$ and $\|\pmb{A}\|_*$ indicate the spectral, Frobenius and nuclear norms of $\pmb{A}$, respectively.
We let $\|\pmb{A}\|_{2,\infty} = \max_{i} \| \pmb{A}_{i,:}\|$.
$\pmb{A}^\top$ is the transpose of $\pmb{A}$.
$\langle \pmb{A}, \pmb{B} \rangle$ and $\pmb{A}\circ\pmb{B}$ represent the inner product and the Hadamard product of $\pmb{A}$ and $\pmb{B}$, respectively.
For any positive integer $n$, we denote $\pmb{1}_{n} = (1,1,\dots,1)^\top\in\mathbb{R}^{n}$
and by $\pmb{I}_n$ the $n$-dimensional identity matrix.
$\{\pmb{e}_i~;~ i \in [n]\}$ indicates the standard basis of $\mathbb{R}^n$ where the dimension $n$ is determined properly by the context.

\section{Related Works on Deterministic Sampling Schemes}
\label{appendix_sec:existing_works}

We discuss several works on the matrix completion problem under deterministic sampling schemes, and the difference from our work.
\cite{pimentel2016characterization} provided algebraic criteria on the sampling pattern under which there are only finitely many low-rank matrices that agree with the underlying matrix over the observed entries.
\cite{shapiro2018matrix} investigated the minimum rank matrix completion problem under deterministic sampling schemes and provided conditions to find a locally-unique solution of the problem.
However, these papers did not give a tractable algorithm to achieve a desirable solution.
Even though \cite{lee2013matrix} and \cite{foucart2020weighted} suggested tractable algorithms for matrix completion under deterministic sampling schemes,
they only derived the error bounds based on a weighted error metric, which is not useful when we want to guarantee accuracy for all entries in the entire matrix. 
\cite{chatterjee2020deterministic} gave conditions for a sequence of the matrix completion problems with arbitrary missing patterns to be asymptotically solvable.
However, our interest lies in the non-asymptotic solvability of a single matrix completion problem.

\section{Proof of Theorems for Symmetric Matrices}

In this section, we consider the case that the underlying rank-$r$ matrix $\pmb{M}$ is symmetric.
We let $\pmb{M} = \pmb{U}\pmb{\Sigma}\pmb{V}^\top$ be the singular value decomposition of $\pmb{M} \in \mathbb{R}^{n\times n}$, where each $\pmb{V}_{:,k}$ is equal to $\pmb{U}_{:,k}$ or $-\pmb{U}_{:,k}$ for $k\in [r]$.
That is, $\pmb{U}\pmb{U}^\top$ equals $\pmb{V}\pmb{V}^\top$.

We introduce the orthogonal decomposition of the space of symmetric matrices, $\tilde{T}\oplus\tilde{T}^\perp$,
where $\tilde{T}$ is span of all matrices of form $\pmb{U}\pmb{X}\pmb{U}^\top$ for an arbitrary symmetric matrix $\pmb{X}\in \mathbb{R}^{r\times r}$.
$\tilde{T}^\perp$ is its orthogonal complement.
Then, the orthogonal projection $P_{\tilde{T}}$ onto $\tilde{T}$ is given by
$$
P_{\tilde{T}}(\pmb{Z}) = \pmb{U}\pmb{U}^\top \pmb{Z} \pmb{U}\pmb{U}^\top = \pmb{U}\pmb{U}^\top \pmb{Z} \pmb{V}\pmb{V}^\top,
$$
and the orthogonal projection onto $\tilde{T}^\perp$ is given by
$$
P_{\tilde{T}^\perp}(\pmb{Z}) = (\pmb{I}_n-\pmb{U}\pmb{U}^\top) \pmb{Z} (\pmb{I}_n-\pmb{U}\pmb{U}^\top) = (\pmb{I}_n-\pmb{U}\pmb{U}^\top) \pmb{Z} (\pmb{I}_n-\pmb{V}\pmb{V}^\top).
$$

\medskip

Before presenting the proofs of Theorems \ref{thm:noiseless_symmetric} and \ref{thm:noisy_symmetric},
we first introduce a lemma which will be used in the proofs.

\begin{lemma}
\label{lemma:symmetric_bound}
Suppose that the assumption \textup{\textbf{A1}} holds. For any matrix $\pmb{W} \in \tilde{T}$, 
$$
\bigg\| \frac{n^2}{|\Omega|} P_{\tilde{T}} P_{\Omega} (\pmb{W}) - \pmb{W} \bigg\|
\leq 
\bigg\| \frac{n^2}{|\Omega|} P_{\Omega} (\pmb{W}) - \pmb{W} \bigg\|
\leq
\alpha \| \pmb{W} \|
$$
where $\alpha = \frac{\mu_0 r n}{|\Omega|}\cdot (2 \xi_{\mathcal{G}}+\psi_{\mathcal{G}})$.
\end{lemma}

\begin{proof}
For any $\pmb{W} \in \tilde{T}$, we can write that $\pmb{W} = \pmb{U}\pmb{U}^\top \pmb{W} \pmb{U}\pmb{U}^\top$. Let $\pmb{X} = \pmb{U}\pmb{U}^\top \pmb{W} \pmb{U}$, i.e., $\pmb{W} = \pmb{U}\pmb{X}^\top$.
Then we can derive that
\begin{align*}
&\bigg\| \frac{n^2}{|\Omega|} P_{\tilde{T}} P_{\Omega} (\pmb{W}) - \pmb{W} \bigg\|
= 
\bigg\| \frac{n^2}{|\Omega|} \pmb{U}\pmb{U}^\top P_{\Omega} (\pmb{W}) \pmb{U}\pmb{U}^\top - \pmb{U}\pmb{U}^\top \pmb{W} \pmb{U}\pmb{U}^\top \bigg\|
\\&\leq
\bigg\| \frac{n^2}{|\Omega|} P_{\Omega} (\pmb{W}) - \pmb{W} \bigg\|
=
\bigg\| \frac{n^2}{|\Omega|} P_{\Omega} (\pmb{U} \pmb{X}^\top) - \pmb{U} \pmb{X}^\top \bigg\|
\leq
\frac{\sqrt{\mu_0 r n^2}}{|\Omega|} \cdot (2\xi_{\mathcal{G}} + \psi_{\mathcal{G}} ) \cdot \sqrt{n} \cdot \| \pmb{X} \|_{2,\infty}
\end{align*}
where the last inequality holds by Lemma \ref{lemma:l2bound_general}.
Also, note that 
\begin{align*}
\| \pmb{X} \|_{2,\infty}^2
&= 
\| \pmb{U}\pmb{U}^\top \pmb{W} \pmb{U} \|_{2,\infty}^2
= 
\max_{i\in [n]} \sum_{k\in[r]} \{ (\pmb{U} \pmb{U}^\top)^\top_{i,:} \pmb{W} \pmb{U}_{:,k} \}^2
=
\max_{i\in [n]} (\pmb{U} \pmb{U}^\top)^\top_{i,:} \pmb{W} 
\pmb{U} \pmb{U}^\top
\pmb{W} (\pmb{U} \pmb{U}^\top)_{i,:}
\\&=
\max_{i\in [n]} \| \pmb{W} (\pmb{U} \pmb{U}^\top)_{i,:} \|^2
\leq
\| \pmb{W} \|^2\cdot \max_{i\in [n]} \|  (\pmb{U} \pmb{U}^\top)_{i,:} \|^2
\leq
\| \pmb{W} \|^2\cdot \frac{\mu_0 r}{n}
\end{align*}
where the last inequality holds by the assumption \textbf{A1}.

Hence,
\begin{align*}
\bigg\| \frac{n^2}{|\Omega|} P_{\tilde{T}} P_{\Omega} (\pmb{W}) - \pmb{W} \bigg\|
\leq 
\bigg\| \frac{n^2}{|\Omega|} P_{\Omega} (\pmb{W}) - \pmb{W} \bigg\|
&\leq
\frac{\sqrt{\mu_0 r n^2}}{|\Omega|} \cdot (2\xi_{\mathcal{G}} + \psi_{\mathcal{G}} ) \cdot \sqrt{n} \cdot \sqrt{\frac{\mu_0 r}{n}}
\| \pmb{W} \|
\\&=
\frac{\mu_0 r n}{|\Omega|} \cdot (2\xi_{\mathcal{G}} + \psi_{\mathcal{G}} ) \cdot \| \pmb{W} \|.
\end{align*}

\end{proof}

\subsection{Proof of Theorem \ref{thm:noiseless_symmetric}}
\label{appendix_subsec:proof_of_thm1}

For any non-zero symmetric matrix $\pmb{Z}\in\mathbb{R}^{n\times n}$ such that $P_{\Omega}(\pmb{Z}) = \pmb{0}$, we want to show the following inequality:
$$
\| \pmb{M} + \pmb{Z} \|_* > \| \pmb{M} \|_*,
$$
so that $\pmb{M}$ is the unique minimizer of the nuclear norm minimization problem \eqref{eq:nnm}.

Choose $\pmb{U}_{\perp}$ and $\pmb{V}_{\perp}$ such that $\langle \pmb{U}_{\perp}\pmb{V}_{\perp}^\top, P_{\tilde{T}^\perp}(\pmb{Z}) \rangle = \| P_{\tilde{T}^\perp}(\pmb{Z}) \|_*$.

For any symmetric matrix $\pmb{Y}\in\mathbb{R}^{n\times n}$ such that $P_{\Omega^c}(\pmb{Y}) = \pmb{0}$, i.e., $\langle \pmb{Y}, \pmb{Z} \rangle = 0$,
\begin{align*}
\| \pmb{M} + \pmb{Z} \|_*
&\geq
\langle \pmb{U}\pmb{V}^\top + \pmb{U}_{\perp}\pmb{V}_{\perp}^\top, \pmb{M} + \pmb{Z} \rangle
=
\| \pmb{M} \|_* + \langle \pmb{U}\pmb{V}^\top + \pmb{U}_{\perp}\pmb{V}_{\perp}^\top - \pmb{Y}, \pmb{Z} \rangle
\\&
=
\| \pmb{M} \|_* + \langle \pmb{U}\pmb{V}^\top + \pmb{U}_{\perp}\pmb{V}_{\perp}^\top - P_{\tilde{T}}(\pmb{Y}) - P_{\tilde{T}^\perp}(\pmb{Y}), P_{\tilde{T}}(\pmb{Z}) + P_{\tilde{T}^\perp}(\pmb{Z}) \rangle
\\&
= \| \pmb{M} \|_* + \langle \pmb{U}\pmb{V}^\top - P_{\tilde{T}}(\pmb{Y}), P_{\tilde{T}}(\pmb{Z}) \rangle
+ \| P_{\tilde{T}^\perp}(\pmb{Z}) \|_* - \langle P_{\tilde{T}^\perp}(\pmb{Y}), P_{\tilde{T}^\perp}(\pmb{Z}) \rangle
\\&
\geq
\| \pmb{M} \|_* - \| \pmb{U}\pmb{V}^\top - P_{\tilde{T}}(\pmb{Y}) \|_* \cdot \| P_{\tilde{T}}(\pmb{Z}) \|
+ \| P_{\tilde{T}^\perp}(\pmb{Z}) \|_* - \| P_{\tilde{T}^\perp}(\pmb{Y}) \| \cdot \| P_{\tilde{T}^\perp}(\pmb{Z}) \|_*,
\end{align*}
where the last inequality holds from the definition of dual norm.

\paragraph{1) Bound of $\| P_{\tilde{T}}(\pmb{Z}) \|$:}
We can derive that 
\begin{align}
\label{eq:A1_bound1}
\| P_{\tilde{T}}(\pmb{Z}) \| 
\leq \frac{n^2}{|\Omega|(1-\alpha)} \cdot \|P_\Omega P_{\tilde{T}} (\pmb{Z})\|
\leq \frac{n^2}{|\Omega|(1-\alpha)} \cdot \|P_\Omega P_{{\tilde{T}}^\perp} (\pmb{Z})\|
\leq \frac{n^2}{|\Omega|(1-\alpha)} \cdot \|P_{{\tilde{T}}^\perp} (\pmb{Z})\|_*
\end{align}
where the first inequality is from Lemma \ref{lemma:symmetric_bound},
the second inequality holds by the fact that $P_{\Omega}(\pmb{Z}) = \pmb{0}$,
and the third inequality holds by the fact that $\|P_\Omega(\pmb{A})\| \leq \|P_\Omega(\pmb{A})\|_F \leq \|\pmb{A}\|_F \leq \|\pmb{A}\|_*$ for any matrix $\pmb{A}$.

\paragraph{2) Bound of $\| \pmb{U}\pmb{V}^\top - P_{\tilde{T}}(\pmb{Y}) \|_*$:}
In a similar way to \cite{recht2011simpler}, we construct $\pmb{Y}$ as follows:
first, we let $\pmb{W}_0 = \pmb{U}\pmb{V}^\top$ and
$$
\pmb{Y}_k = \frac{n^2}{|\Omega|}P_{\Omega} \bigg(\sum_{i=0}^{k-1}\pmb{W}_{i}\bigg),
~~ \pmb{W}_k = \pmb{U}\pmb{V}^\top - P_{\tilde{T}}(\pmb{Y}_k)
$$
for $k= 1,2,\cdots$.
Then, we can check that
$$
\pmb{W}_k - \frac{n^2}{|\Omega|}P_{\tilde{T}} P_\Omega (\pmb{W}_k)
= \pmb{W}_k - P_{\tilde{T}}(\pmb{Y}_{k+1} - \pmb{Y}_k)
= \pmb{U}\pmb{V}^\top - P_{\tilde{T}} (\pmb{Y}_{k+1})
= \pmb{W}_{k+1}.
$$
Hence, by Lemma \ref{lemma:symmetric_bound}, we derive that
$$
\| \pmb{W}_{k+1} \| 
= 
\| \pmb{W}_k - \frac{n^2}{|\Omega|}P_{\tilde{T}} P_\Omega (\pmb{W}_k) \|
\leq
\alpha \| \pmb{W}_k \|
\leq
\cdots
\leq
\alpha^{k+1} \| \pmb{W}_0 \|
= \alpha^{k+1}.
$$
Let $\pmb{Y} = \pmb{Y}_m$ for a sufficiently large $m$. Then, we have that
\begin{equation}
\label{eq:A1_bound2}
\| \pmb{U}\pmb{V}^\top - P_{\tilde{T}}(\pmb{Y}) \|_*
= \| \pmb{W}_m \|_*
\leq r \| \pmb{W}_m \|
\leq r \alpha^{m}.
\end{equation}

\paragraph{3) Bound of $\| P_{\tilde{T}^\perp}(\pmb{Y}) \|$:}
Since $P_{\tilde{T}}(\pmb{Y}_m) = \sum_{k=0}^{m-1} P_{\tilde{T}}(\pmb{Y}_{k+1} - \pmb{Y}_k) = \sum_{k=0}^{m-1} (\pmb{W}_k - \pmb{W}_{k+1})$,
we can derive that
\begin{align*}
\| P_{\tilde{T}^\perp}(\pmb{Y}) \|
&= 
\| \pmb{Y}_m - P_{\tilde{T}} (\pmb{Y}_m) \|
=
\bigg\| \frac{n^2}{|\Omega|} P_{\Omega}\bigg(\sum_{k=0}^{m-1} \pmb{W}_k\bigg)
- \sum_{k=0}^{m-1} (\pmb{W}_k - \pmb{W}_{k+1}) \bigg\|
\\&=
\bigg\| \sum_{k=0}^{m-1} \pmb{W}_{k+1} - 
\sum_{k=0}^{m-1} \bigg\{ \pmb{W}_k - \frac{n^2}{|\Omega|} P_{\Omega}(\pmb{W}_k) \bigg\} \bigg\|
\leq
\sum_{k=0}^{m-1} \| \pmb{W}_{k+1} \|
+ \sum_{k=0}^{m-1} \bigg\| \pmb{W}_k - \frac{n^2}{|\Omega|} P_{\Omega}(\pmb{W}_k) \bigg\|.
\end{align*}
Since $\| \pmb{W}_{k+1} \| \leq \alpha^{k+1}$ and $\big\| \pmb{W}_k - \frac{n^2}{|\Omega|} P_{\Omega}(\pmb{W}_k) \big\| \leq \alpha^{k+1}$ by Lemma \ref{lemma:symmetric_bound},
\begin{equation}
\label{eq:A1_bound3}
\| P_{\tilde{T}^\perp}(\pmb{Y}) \|
\leq 2\cdot \sum_{k=0}^{m-1} \alpha^{k+1}
= 2\cdot \frac{1-\alpha^m}{1-\alpha}\cdot \alpha.
\end{equation}

Using \eqref{eq:A1_bound1}, \eqref{eq:A1_bound2} and \eqref{eq:A1_bound3}, we derive the following inequality:
\begin{align}
\| \pmb{M} + \pmb{Z} \|_*
&\geq
\| \pmb{M} \|_* - r\alpha^m \cdot \frac{n^2}{|\Omega|(1-\alpha)}\cdot \| P_{\tilde{T}^\perp}(\pmb{Z}) \|_*
+ \| P_{\tilde{T}^\perp}(\pmb{Z}) \|_* - 2\cdot \frac{1-\alpha^m}{1-\alpha}\cdot \alpha \cdot \| P_{\tilde{T}^\perp}(\pmb{Z}) \|_*
\nonumber \\&=
\| \pmb{M} \|_* + \| P_{\tilde{T}^\perp}(\pmb{Z}) \|_* \cdot
\bigg\{
1 - \frac{2\alpha}{1-\alpha}
- \frac{\alpha^m}{1-\alpha} \cdot \bigg( \frac{rn^2}{|\Omega|} - 2\alpha \bigg) 
\bigg\}.
\label{eq:A1_bound4}
\end{align}
If $\alpha < \frac{1}{3}$, then $1 - \frac{2\alpha}{1-\alpha} > 0$ and we can let $m$ sufficiently large so that 
$\frac{\alpha^m}{1-\alpha} \cdot \big( \frac{rn^2}{|\Omega|} - 2\alpha \big) < 1 - \frac{2\alpha}{1-\alpha}$.

Therefore, if $\alpha < \frac{1}{3}$, i.e., 
$|\Omega| >3 \mu_0 r n\cdot (2 \xi_{\mathcal{G}}+\psi_{\mathcal{G}})$,
then
$\| \pmb{M} + \pmb{Z} \|_* > \| \pmb{M} \|_*$,
that is, 
$\pmb{M}$ is the unique minimizer of the nuclear norm minimization problem \eqref{eq:nnm}.

\subsection{Proof of Theorem \ref{thm:noisy_symmetric}}
\label{appendix_subsec:proof_of_thm2}

Note that
$$
\| \hat{\pmb{M}} - \pmb{M} \|_F^2
=
\| P_{\Omega} (\hat{\pmb{M}} - \pmb{M} ) \|_F^2
+ \| P_{\Omega^c} (\hat{\pmb{M}} - \pmb{M} ) \|_F^2.
$$

\paragraph{1) Bound of $\| P_{\Omega^c} (\hat{\pmb{M}} - \pmb{M} ) \|_F^2$:}
Let $\pmb{Z} = P_{\Omega^c} ( \hat{\pmb{M}} - \pmb{M} )$.
Note that $\|\pmb{Z} \|_F^2 = \| P_{\tilde{T}} (\pmb{Z}) \|_F^2 + \| P_{\tilde{T}^\perp} (\pmb{Z}) \|_F^2$.
By Lemma \ref{lemma:symmetric_bound} and the fact that $\|\pmb{A}\| \leq \|\pmb{A}\|_F \leq \sqrt{r} \|\pmb{A}\|$ for any rank-$r$ matrix $\pmb{A}$,
we have that
$$
\bigg\| \frac{n^2}{|\Omega|} P_{\tilde{T}} P_{\Omega} P_{\tilde{T}} (\pmb{Z}) - P_{\tilde{T}} (\pmb{Z}) \bigg\|_F
\leq 
\sqrt{r} \cdot \bigg\| \frac{n^2}{|\Omega|} P_{\tilde{T}} P_{\Omega} P_{\tilde{T}} (\pmb{Z}) - P_{\tilde{T}} (\pmb{Z}) \bigg\|
\leq
\sqrt{r} \alpha \| P_{\tilde{T}} (\pmb{Z})\|
\leq
\sqrt{r} \alpha \| P_{\tilde{T}} (\pmb{Z})\|_F.
$$
Then we can derive
\begin{align*}
\| P_{\tilde{T}}(\pmb{Z}) \|_F^2
\leq 
\frac{n^2}{|\Omega|(1-\sqrt{r}\alpha)} \cdot 
\| P_{\Omega} P_{\tilde{T}}(\pmb{Z}) \|_F^2
\leq
\frac{n^2}{|\Omega|(1-\sqrt{r}\alpha)} \cdot 
\| P_{\Omega} P_{\tilde{T}^\perp}(\pmb{Z}) \|_F^2
\leq
\frac{n^2}{|\Omega|(1-\sqrt{r}\alpha)} \cdot 
\| P_{\tilde{T}^\perp}(\pmb{Z}) \|_F^2
\end{align*}
where the first inequality holds by Lemma \ref{lemma:frobenius_bound},
the second inequality holds by the fact that $P_{\Omega}(\pmb{Z}) = \pmb{0}$,
and the third inequality holds by the fact that $\|P_\Omega(\pmb{A})\|_F \leq \|\pmb{A}\|_F$ for any matrix $\pmb{A}$.

Also, by using \eqref{eq:A1_bound4} and the fact that 
$\| \pmb{M} + \pmb{Z} \|_* 
= 
\| P_{\Omega} (\pmb{M}) + P_{\Omega^c} (\pmb{M}) +  P_{\Omega^c} (\hat{\pmb{M}} - \pmb{M} ) \|_*
=
\| P_{\Omega} (\pmb{M} - \hat{\pmb{M}}) + \hat{\pmb{M}} \|_*
\leq
\| P_{\Omega} (\pmb{M} - \hat{\pmb{M}}) \|_* + \| \hat{\pmb{M}} \|_*
$, 
we can derive
$$
\| P_{\tilde{T}^\perp}(\pmb{Z}) \|_*
\leq
\| P_{\Omega} (\pmb{M} - \hat{\pmb{M}}) \|_*
\cdot 
\bigg\{
1 - \frac{2\alpha}{1-\alpha}
- \frac{\alpha^m}{1-\alpha} \cdot \bigg( \frac{rn^2}{|\Omega|} - 2\alpha \bigg) 
\bigg\}^{-1}.
$$
Therefore,
\begin{align*}
&\| P_{\Omega^c} (\hat{\pmb{M}} - \pmb{M} ) \|_F^2
= 
\| \pmb{Z} \|_F^2
\\&\leq
\frac{n^2}{|\Omega|(1-\sqrt{r}\alpha)} \cdot 
\| P_{\tilde{T}^\perp}(\pmb{Z}) \|_F^2 + \| P_{\tilde{T}^\perp} (\pmb{Z}) \|_F^2
\\&\leq
\bigg(1+\frac{n^2}{|\Omega|(1-\sqrt{r}\alpha)}\bigg)
\cdot
\| P_{\tilde{T}^\perp}(\pmb{Z}) \|_*^2
\\&\leq
\bigg(1+\frac{n^2}{|\Omega|(1-\sqrt{r}\alpha)}\bigg)
\cdot 
\bigg\{
1 - \frac{2\alpha}{1-\alpha}
- \frac{\alpha^m}{1-\alpha} \cdot \bigg( \frac{rn^2}{|\Omega|} - 2\alpha \bigg) 
\bigg\}^{-2}
\cdot
\| P_{\Omega} (\pmb{M} - \hat{\pmb{M}}) \|_*^2
\end{align*}

\paragraph{2) Bound of $\| P_{\Omega} (\hat{\pmb{M}} - \pmb{M} ) \|_F$:}
First, note that 
$$
\| P_{\Omega} (\hat{\pmb{M}} - \pmb{M} ) \|_F
\leq 
\| P_{\Omega} (\hat{\pmb{M}} - \pmb{M} - \pmb{E}) \|_F
+
\| P_{\Omega} (\pmb{E}) \|_F
\leq 
\delta + \| P_{\Omega} (\pmb{E}) \|_F
$$
where the last inequality holds by the constraint of the problem \eqref{eq:cnnm}.

By using the tail bound of norm of random vector of independent sub-Gaussian variables (e.g., Theorem 3.1.1. in \cite{vershynin2018high}),
we can show that
$$
\| P_{\Omega} (\pmb{E}) \|_F
\leq
4\sigma\sqrt{|\Omega|} + 2\sigma\sqrt{\log (\eta^{-1})}
\leq
\delta
$$
with probability at least $1- \eta$.

Now, we have the upper bound of $\| \hat{\pmb{M}} - \pmb{M} \|_F^2$ as follows:
\begin{align*}
&\| \hat{\pmb{M}} - \pmb{M} \|_F^2
\leq
\| P_{\Omega} (\hat{\pmb{M}} - \pmb{M} ) \|_F^2
+
\bigg(1+\frac{n^2}{|\Omega|(1-\sqrt{r}\alpha)}\bigg)
\cdot 
\bigg\{
1 - \frac{2\alpha}{1-\alpha}
- \frac{\alpha^m}{1-\alpha} \cdot \bigg( \frac{rn^2}{|\Omega|} - 2\alpha \bigg) 
\bigg\}^{-2}
\cdot
\| P_{\Omega} (\pmb{M} - \hat{\pmb{M}}) \|_*^2
\\&\leq
\bigg[ 1 + 
n \cdot \bigg(1+\frac{n^2}{|\Omega|(1-\sqrt{r}\alpha)}\bigg)
\cdot 
\bigg\{
1 - \frac{2\alpha}{1-\alpha}
- \frac{\alpha^m}{1-\alpha} \cdot \bigg( \frac{rn^2}{|\Omega|} - 2\alpha \bigg) 
\bigg\}^{-2}
\bigg]
\cdot
\| P_{\Omega} (\hat{\pmb{M}} - \pmb{M} ) \|_F^2
\\&\leq
\bigg[ 1 + 
n \cdot \bigg(1+\frac{n^2}{|\Omega|(1-\sqrt{r}\alpha)}\bigg)
\cdot 
\bigg\{
1 - \frac{2\alpha}{1-\alpha}
- \frac{\alpha^m}{1-\alpha} \cdot \bigg( \frac{rn^2}{|\Omega|} - 2\alpha \bigg) 
\bigg\}^{-2}
\bigg]
\cdot
(2\delta)^2.
\end{align*}
If $\frac{2\alpha}{1-\alpha} < 1-c \leq 1$ for some positive constant $c$, then by letting $m$ be sufficiently large, we have that 
$ 
1 - \frac{2\alpha}{1-\alpha}
- \frac{\alpha^m}{1-\alpha} \cdot \big( \frac{rn^2}{|\Omega|} - 2\alpha \big)
> c
$.
Furthermore, if $\sqrt{r}\alpha < 1-c'$ for some positive constant $c'$, then we derive the following inequality:
$$
\| \hat{\pmb{M}} - \pmb{M} \|_F
\leq
2\delta + 2\delta \cdot \sqrt{\frac{n}{c^2}\cdot \{ 1 + (c'p)^{-1}\}}
$$
where $p = \frac{|\Omega|}{n^2}$.
Then we obtain the desired result by letting $c=\frac{1}{2}$.

\section{Proof of Theorems for Rectangular Matrices}
\label{appendix_sec:proof_rectangular}

Here, we consider the case of general rectangular rank-$r$ matrices.
We let $\pmb{M} = \pmb{U}\pmb{\Sigma}\pmb{V}^\top$ be the singular value decomposition of the underlying rank-$r$ matrix $\pmb{M} \in \mathbb{R}^{n_1\times n_2}$.

We introduce the orthogonal decomposition of the space of $n_1 \times n_2$ matrices, ${T}\oplus{T}^\perp$,
where ${T}$ is span of all matrices of form $\pmb{U}\pmb{X}^\top$ and $\pmb{Y}\pmb{V}^\top$ for arbitrary matrices $\pmb{X} \in \mathbb{R}^{n_2\times r}$ and $\pmb{Y}\in \mathbb{R}^{n_1 \times r}$.
${T}^\perp$ is its orthogonal complement.
Then, the orthogonal projection $P_{{T}}$ onto ${T}$ is given by
$$
P_{{T}}(\pmb{Z}) = \pmb{U}\pmb{U}^\top \pmb{Z} 
+ (\pmb{I}_{n_1} - \pmb{U}\pmb{U}^\top ) \pmb{Z} \pmb{V}\pmb{V}^\top,
$$
and the orthogonal projection onto ${T}^\perp$ is given by
$$
P_{{T}^\perp}(\pmb{Z}) = (\pmb{I}_{n_1}-\pmb{U}\pmb{U}^\top) \pmb{Z} (\pmb{I}_{n_2}-\pmb{V}\pmb{V}^\top).
$$

\medskip

Below is a lemma which will be used in the proofs of Theorems \ref{thm:noiseless_rectangular} and \ref{thm:noisy_rectangular}.

\begin{lemma}
\label{lemma:rectangular_bound}
Suppose that the assumptions \textup{\textbf{A1}} and \textup{\textbf{A2}} hold. 
For any $\pmb{W} \in T$,
let $\pmb{X} = \pmb{W}^\top \pmb{U}$ and $\pmb{Y} = (\pmb{I} - \pmb{U}\pmb{U}^\top) \pmb{W} \pmb{V}$, i.e., 
$\pmb{W} = \pmb{U}\pmb{X}^\top + \pmb{Y}\pmb{V}^\top$.

Consider $\tilde{\pmb{W}} = \pmb{W} - \frac{n_1 n_2}{|\Omega|} P_T P_{\Omega} (\pmb{W})$.
Note that 
\begin{align*}
\tilde{\pmb{W}} 
&= 
\pmb{U}\pmb{X}^\top + \pmb{Y}\pmb{V}^\top - \frac{n_1 n_2}{|\Omega|}\cdot \{ \pmb{U}\pmb{U}^\top P_{\Omega} (\pmb{W}) + (\pmb{I}_{n_1} - \pmb{U} \pmb{U}^\top ) P_{\Omega} (\pmb{W}) \pmb{V} \pmb{V}^\top \}
\\&=
\pmb{U} \bigg\{ \pmb{X}^\top - \frac{n_1 n_2}{|\Omega|}\pmb{U}^\top P_{\Omega}(\pmb{W}) \bigg\}
+ \bigg\{ \pmb{Y} - \frac{n_1 n_2}{|\Omega|} (\pmb{I}_{n_1} - \pmb{U} \pmb{U}^\top ) P_{\Omega} (\pmb{W}) \pmb{V} \bigg\} \pmb{V}^\top.
\end{align*}
Let 
$\tilde{\pmb{X}} = \pmb{X} - \frac{n_1 n_2}{|\Omega|} P_{\Omega}(\pmb{W})^\top \pmb{U}$ 
and 
$\tilde{\pmb{Y}} = \pmb{Y} - \frac{n_1 n_2}{|\Omega|} (\pmb{I}_{n_1} - \pmb{U} \pmb{U}^\top ) P_{\Omega} (\pmb{W}) \pmb{V}$,
i.e., $\tilde{\pmb{W}} = \pmb{U}\tilde{\pmb{X}}^\top + \tilde{\pmb{Y}}\pmb{V}^\top$.
Then,
$$
\| \tilde{\pmb{W}} \|_F^2 
= 
\| \tilde{\pmb{X}} \|_F^2 + \| \tilde{\pmb{Y}} \|_F^2
\leq
(\theta + \gamma)^2 \cdot ( \| \pmb{X} \|_F^2 +\| \pmb{Y} \|_F^2)
= (\theta + \gamma)^2 \cdot \| \pmb{W} \|_F^2
$$
where
$\gamma = \frac{\sqrt{n_1 n_2}}{|\Omega|}\cdot \mu_0 r \cdot ( \xi_{1,\mathcal{G}} + \xi_{2,\mathcal{G}} + \psi_{\mathcal{G}} )$.
When either $\pmb{X}$ or $\pmb{Y}$ is $\pmb{0}$, the following holds:
$$
\| \tilde{\pmb{W}} \|_F^2 
\leq (\theta^2 + \gamma^2)\cdot \| \pmb{W} \|_F^2.
$$

\end{lemma}

\begin{proof}
Note that
$ \| \tilde{\pmb{X}} \|_F 
= 
\| \pmb{X} - \frac{n_1 n_2}{|\Omega|} P_{\Omega}( \pmb{U}\pmb{X}^\top + \pmb{Y}\pmb{V}^\top )^\top \pmb{U} \|_F
\leq
\| \pmb{X} - \frac{n_1 n_2}{|\Omega|} P_{\Omega}( \pmb{U}\pmb{X}^\top )^\top \pmb{U} \|_F
+
\frac{n_1 n_2}{|\Omega|} \cdot \| P_{\Omega}( \pmb{Y}\pmb{V}^\top )^\top \pmb{U} \|_F.
$

\paragraph{1) Bound of $\| \pmb{X} - \frac{n_1 n_2}{|\Omega|} P_{\Omega}( \pmb{U}\pmb{X}^\top )^\top \pmb{U} \|_F$:}
\begin{align*}
\bigg\| \pmb{X} - \frac{n_1 n_2}{|\Omega|} P_{\Omega}( \pmb{U}\pmb{X}^\top )^\top \pmb{U} \bigg\|_F^2
&=
\sum_{j \in [n_2]} \sum_{k \in [r]} 
\bigg( X_{j,k} - \frac{n_1 n_2}{|\Omega|} \sum_{i\in [n_1]} U_{i,k} A_{i,j} \pmb{U}_{i,:}^\top \pmb{X}_{j,:} \bigg)^2
\\&=
\sum_{j \in [n_2]} \sum_{k \in [r]} 
\bigg( \pmb{e}_k^\top \pmb{X}_{j,:} - \frac{n_1 n_2}{|\Omega|} \sum_{i\in [n_1]} \pmb{e}_k^\top \pmb{U}_{i,:} A_{i,j} \pmb{U}_{i,:}^\top \pmb{X}_{j,:} \bigg)^2
\\&=
\sum_{j \in [n_2]} \sum_{k \in [r]} 
\bigg\{ \pmb{e}_k^\top 
\bigg( \pmb{I}_r - \frac{n_1 n_2}{|\Omega|} \sum_{i\in [n_1]} A_{i,j} \pmb{U}_{i,:} \pmb{U}_{i,:}^\top \bigg)
\pmb{X}_{j,:} \bigg\}^2
\\&=
\sum_{j \in [n_2]} \pmb{X}_{j,:}^\top 
\bigg( \pmb{I}_r - \frac{n_1 n_2}{|\Omega|} \sum_{i\in [n_1]} A_{i,j} \pmb{U}_{i,:} \pmb{U}_{i,:}^\top \bigg)^2 \pmb{X}_{j,:}
\\&\leq
\sum_{j \in [n_2]}
\bigg\| \pmb{I}_r - \frac{n_1 n_2}{|\Omega|} \sum_{i\in [n_1]} A_{i,j} \pmb{U}_{i,:} \pmb{U}_{i,:}^\top \bigg\|^2 \cdot
\| \pmb{X}_{j,:} \|^2
\\&\leq
\theta^2 \cdot \| \pmb{X} \|_F^2
\end{align*}
where the last inequality holds by the assumption \textbf{A2}.

\paragraph{2) Bound of $\| P_{\Omega}( \pmb{Y}\pmb{V}^\top )^\top \pmb{U} \|_F$:}
\begin{align*}
\| P_{\Omega}( \pmb{Y}\pmb{V}^\top )^\top \pmb{U} \|_F^2
&=
\sup_{\pmb{y}\in\mathbb{R}^{n_2}, \|\pmb{y}\|=1}
\sum_{k\in[r]}
\bigg\{ \sum_{j\in[n_2]} \sum_{i\in[n_1]} \sum_{l\in[r]} U_{i,k} A_{i,j} Y_{i,l} V_{j,l} y_j \bigg\}^2
\\&=
\sup_{\pmb{y}\in\mathbb{R}^{n_2}, \|\pmb{y}\|=1}
\sum_{k\in[r]}
\bigg\{ \sum_{l\in[r]} (\pmb{U}_{:,k}\circ \pmb{Y}_{:,l})^\top 
\pmb{A} (\pmb{V}_{:,l}\circ \pmb{y}) \bigg\}^2
\\&\leq
\sup_{\pmb{y}\in\mathbb{R}^{n_2}, \|\pmb{y}\|=1}
\sum_{k\in[r]}
\bigg\{ \sum_{l\in[r]} \| \pmb{U}_{:,k}\circ \pmb{Y}_{:,l} \|\cdot 
\|\pmb{V}_{:,l}\circ \pmb{y}\| \cdot (\xi_{1,\mathcal{G}} + \psi_{\mathcal{G}}) \bigg\}^2
\\&\leq
\sup_{\pmb{y}\in\mathbb{R}^{n_2}, \|\pmb{y}\|=1}
\sum_{k\in[r]}
\bigg( \sum_{l\in[r]} \sum_{i\in[n_1]} U_{i,k}^2 Y_{i,l}^2 \bigg)
\cdot 
\bigg( \sum_{l\in[r]} \sum_{j\in[n_2]} V_{j,l}^2 y_{j}^2 \bigg)
\cdot (\xi_{1,\mathcal{G}} + \psi_{\mathcal{G}})^2
\\&\leq
\frac{\mu_0 r}{n_1}\cdot \|\pmb{Y}\|_F^2 
\cdot 
\frac{\mu_0 r}{n_2} \cdot (\xi_{1,\mathcal{G}} + \psi_{\mathcal{G}})^2
=
\frac{(\mu_0 r)^2}{n_1 n_2} \cdot (\xi_{1,\mathcal{G}} + \psi_{\mathcal{G}})^2 \cdot \|\pmb{Y}\|_F^2
\end{align*}
where the first inequality holds by Lemma \ref{lemma:adjacency_matrix} and the fact that $(\pmb{U}_{:,k}\circ \pmb{Y}_{:,l})^\top \pmb{1}_{n_1} = \pmb{U}_{:,k}^\top \pmb{Y}_{:,l} = 0$, the second inequality holds by the Cauchy–Schwarz inequality, and the last inequality holds by the assumption \textbf{A1}.

Therefore,
\begin{equation}
\label{eq:C_bound1}
\| \tilde{\pmb{X}} \|_F
\leq
\theta \cdot \|\pmb{X}\|_F
+ \frac{\sqrt{n_1 n_2}}{|\Omega|} \cdot (\mu_0 r) \cdot (\xi_{1,\mathcal{G}} + \psi_{\mathcal{G}}) \cdot \|\pmb{Y}\|_F.
\end{equation}

\medskip

Next, $ \| \tilde{\pmb{Y}} \|_F 
= 
\| \pmb{Y} - \frac{n_1 n_2}{|\Omega|} (\pmb{I}_{n_1} - \pmb{U} \pmb{U}^\top ) P_{\Omega} (\pmb{U}\pmb{X}^\top + \pmb{Y}\pmb{V}^\top) \pmb{V} \|_F
\leq
\| \pmb{Y} - \frac{n_1 n_2}{|\Omega|} (\pmb{I}_{n_1} - \pmb{U} \pmb{U}^\top ) P_{\Omega} (\pmb{Y}\pmb{V}^\top) \pmb{V} \|_F
+
\frac{n_1 n_2}{|\Omega|} \| (\pmb{I}_{n_1} - \pmb{U} \pmb{U}^\top ) P_{\Omega} (\pmb{U}\pmb{X}^\top) \pmb{V} \|_F$.

\paragraph{3) Bound of $\| \pmb{Y} - \frac{n_1 n_2}{|\Omega|} (\pmb{I}_{n_1} - \pmb{U} \pmb{U}^\top ) P_{\Omega} (\pmb{Y}\pmb{V}^\top) \pmb{V} \|_F$:}
\begin{align*}
&\bigg\| \pmb{Y} - \frac{n_1 n_2}{|\Omega|} (\pmb{I}_{n_1} - \pmb{U} \pmb{U}^\top ) P_{\Omega} (\pmb{Y}\pmb{V}^\top) \pmb{V} \bigg\|_F^2
=
\bigg\|(\pmb{I}_{n_1} - \pmb{U} \pmb{U}^\top )\cdot \bigg\{ \pmb{Y} - \frac{n_1 n_2}{|\Omega|}  P_{\Omega} (\pmb{Y}\pmb{V}^\top) \pmb{V} \bigg\} \bigg\|_F^2
\\&=
\sup_{\pmb{y}\in\mathbb{R}^{n_2}, \|\pmb{y}\|=1}
\sum_{k\in[r]}
\bigg\{
\sum_{i\in[n_1]}
y_i 
\sum_{m\in [n_1]}
(\pmb{I}_{n_1} - \pmb{U} \pmb{U}^\top )_{i,m}
\bigg(
Y_{m,k} - 
\frac{n_1 n_2}{|\Omega|} \sum_{j\in[n_2]}
V_{j,k} A_{m,j} \pmb{Y}_{m,:}^\top \pmb{V}_{j,:}
\bigg) \bigg\}^2
\\&=
\sup_{\pmb{y}\in\mathbb{R}^{n_2}, \|\pmb{y}\|=1}
\sum_{k\in[r]}
\bigg\{
\sum_{m\in [n_1]}
\pmb{y}^\top (\pmb{I}_{n_1} - \pmb{U} \pmb{U}^\top )_{:,m}
\bigg(
\pmb{e}_k^\top \pmb{Y}_{m,:} - 
\frac{n_1 n_2}{|\Omega|} \sum_{j\in[n_2]}
\pmb{e}_k^\top \pmb{V}_{j,:} A_{m,j} \pmb{Y}_{m,:}^\top \pmb{V}_{j,:}
\bigg) \bigg\}^2
\\&=
\sup_{\pmb{y}\in\mathbb{R}^{n_2}, \|\pmb{y}\|=1}
\sum_{k\in[r]}
\bigg\{
\sum_{m\in [n_1]}
\pmb{y}^\top (\pmb{I}_{n_1} - \pmb{U} \pmb{U}^\top )_{:,m}
\cdot
\pmb{e}_k^\top \bigg( \pmb{I}_r
 - 
\frac{n_1 n_2}{|\Omega|} \sum_{j\in[n_2]}
A_{m,j} \pmb{V}_{j,:} \pmb{V}_{j,:}^\top
\bigg)\pmb{Y}_{m,:} \bigg\}^2
\\&\leq
\sup_{\pmb{y}\in\mathbb{R}^{n_2}, \|\pmb{y}\|=1}
\sum_{k\in[r]}
\bigg[
\sum_{m\in [n_1]}
\Big\{
\pmb{y}^\top (\pmb{I}_{n_1} - \pmb{U} \pmb{U}^\top )_{:,m}
\Big\}^2 \bigg]
\cdot
\bigg[
\sum_{m\in [n_1]}
\Big\{
\pmb{e}_k^\top \bigg( \pmb{I}_r
 - 
\frac{n_1 n_2}{|\Omega|} \sum_{j\in[n_2]}
A_{m,j} \pmb{V}_{j,:} \pmb{V}_{j,:}^\top
\bigg)\pmb{Y}_{m,:} \Big\}^2 \bigg]
\\&=
\sup_{\pmb{y}\in\mathbb{R}^{n_2}, \|\pmb{y}\|=1}
\pmb{y}^\top (\pmb{I}_{n_1} - \pmb{U} \pmb{U}^\top ) \pmb{y}
\cdot
\sum_{m\in [n_1]}
\sum_{k\in[r]}
\Big\{
\pmb{e}_k^\top \bigg( \pmb{I}_r
 - 
\frac{n_1 n_2}{|\Omega|} \sum_{j\in[n_2]}
A_{m,j} \pmb{V}_{j,:} \pmb{V}_{j,:}^\top
\bigg)\pmb{Y}_{m,:} \Big\}^2
\\&\leq
\| \pmb{I}_{n_1} - \pmb{U} \pmb{U}^\top\|
\cdot
\sum_{m\in [n_1]}
\bigg\| \pmb{I}_r - \frac{n_1 n_2}{|\Omega|} \sum_{j\in[n_2]}
A_{m,j} \pmb{V}_{j,:} \pmb{V}_{j,:}^\top \bigg\|^2 \cdot
\| \pmb{Y}_{m,:} \|^2
\\&\leq
\theta^2 \cdot \| \pmb{Y} \|_F^2
\end{align*}
where the last inequality holds by the assumption \textbf{A2}.

\paragraph{4) Bound of $\| (\pmb{I}_{n_1} - \pmb{U} \pmb{U}^\top ) P_{\Omega} (\pmb{U}\pmb{X}^\top) \pmb{V} \|_F^2$:}
\begin{align*}
&\| (\pmb{I}_{n_1} - \pmb{U} \pmb{U}^\top ) P_{\Omega} (\pmb{U}\pmb{X}^\top) \pmb{V} \|_F^2
=
\sup_{\pmb{y}\in\mathbb{R}^{n_2}, \|\pmb{y}\|=1}
\sum_{k\in[r]}
\bigg\{
\sum_{i\in[n_1]}
y_i 
\sum_{m\in [n_1]} \sum_{j\in [n_2]} \sum_{l\in [r]}
(\pmb{I}_{n_1} - \pmb{U} \pmb{U}^\top)_{i,m}
A_{m,j} U_{m,l} X_{j,l} V_{j,k}
\bigg\}^2
\\&=
\sup_{\pmb{y}\in\mathbb{R}^{n_2}, \|\pmb{y}\|=1}
\sum_{k\in[r]}
\bigg[
\sum_{l\in [r]} \bigg\{
\sum_{i\in[n_1]}
y_i (\pmb{I}_{n_1} - \pmb{U} \pmb{U}^\top)_{i,:}\circ \pmb{U}_{:,l} \bigg\}^\top
\pmb{A} \{ \pmb{X}_{:,l} \circ \pmb{V}_{:,k} \} \bigg]^2
\\&\leq
\sup_{\pmb{y}\in\mathbb{R}^{n_2}, \|\pmb{y}\|=1}
\sum_{k\in[r]}
\bigg\{
\sum_{l\in [r]} 
\bigg\| \sum_{i\in[n_1]} y_i (\pmb{I}_{n_1} - \pmb{U} \pmb{U}^\top)_{i,:}\circ \pmb{U}_{:,l} \bigg\| \cdot 
\| \pmb{X}_{:,l} \circ \pmb{V}_{:,k} \| \cdot (\xi_{1,\mathcal{G}} + \psi_{\mathcal{G}}) \bigg\}^2
\\&\leq
(\xi_{1,\mathcal{G}} + \psi_{\mathcal{G}})^2 \cdot
\sup_{\pmb{y}\in\mathbb{R}^{n_2}, \|\pmb{y}\|=1}
\sum_{l\in [r]} \bigg\| \sum_{i\in[n_1]} y_i (\pmb{I}_{n_1} - \pmb{U} \pmb{U}^\top)_{i,:}\circ \pmb{U}_{:,l} \bigg\|^2
\cdot 
\sum_{k\in[r]} \sum_{l\in [r]} \| \pmb{X}_{:,l} \circ \pmb{V}_{:,k} \|^2
\\&=
(\xi_{1,\mathcal{G}} + \psi_{\mathcal{G}})^2 \cdot
\sup_{\pmb{y}\in\mathbb{R}^{n_2}, \|\pmb{y}\|=1}
\sum_{l\in [r]} \sum_{m\in[n_1]}
\bigg\{ \sum_{i\in[n_1]} y_i (\pmb{I}_{n_1} - \pmb{U} \pmb{U}^\top)_{i,m} U_{m,l} \bigg\}^2
\cdot 
\sum_{k\in[r]} \sum_{l\in [r]} \sum_{j\in[n_2]} X_{j,l}^2 V_{j,k}^2
\\&=
(\xi_{1,\mathcal{G}} + \psi_{\mathcal{G}})^2 \cdot \bigg\{
\sup_{\pmb{y}\in\mathbb{R}^{n_2}, \|\pmb{y}\|=1}
\sum_{m\in[n_1]} \sum_{l\in [r]} U_{m,l}^2
\cdot \pmb{y}^\top (\pmb{I}_{n_1} - \pmb{U} \pmb{U}^\top)_{:,m}
(\pmb{I}_{n_1} - \pmb{U} \pmb{U}^\top)_{m,:} \pmb{y} \bigg\}
\cdot 
\sum_{l\in [r]} \sum_{j\in[n_2]} X_{j,l}^2 \sum_{k\in[r]} V_{j,k}^2
\\&\leq
(\xi_{1,\mathcal{G}} + \psi_{\mathcal{G}})^2 \cdot \frac{\mu_0^2 r^2}{n_1 n_2} \cdot \| \pmb{X} \|_F^2
\cdot
\sup_{\pmb{y}\in\mathbb{R}^{n_2}, \|\pmb{y}\|=1}
\pmb{y}^\top (\pmb{I}_{n_1} - \pmb{U} \pmb{U}^\top) \pmb{y}
= 
\frac{(\mu_0 r)^2}{n_1 n_2} \cdot (\xi_{1,\mathcal{G}} + \psi_{\mathcal{G}})^2 \cdot \|\pmb{X}\|_F^2
\end{align*}
where the first inequality holds by Lemma \ref{lemma:adjacency_matrix} and the fact that $\{ (\pmb{I}_{n_1} - \pmb{U} \pmb{U}^\top)_{i,:}\circ \pmb{U}_{:,l} \}^\top \pmb{1}_{n_1} = 0$, the second inequality holds by the Cauchy–Schwarz inequality, and the third inequality holds by the assumption \textbf{A1}.

Therefore,
\begin{equation}
\label{eq:C_bound2}
\| \tilde{\pmb{Y}} \|_F
\leq
\theta \cdot \|\pmb{Y}\|_F
+ \frac{\sqrt{n_1 n_2}}{|\Omega|} \cdot (\mu_0 r) \cdot (\xi_{1,\mathcal{G}} + \psi_{\mathcal{G}}) \cdot \|\pmb{X}\|_F.
\end{equation}

By using \eqref{eq:C_bound1} and \eqref{eq:C_bound2}, we can derive the desired results.

\end{proof}

\subsection{Proof of Theorem \ref{thm:noiseless_rectangular}}
\label{appendix_subsec:proof_of_thm3}

The overall procedure is similar to the proof of Theorem \ref{thm:noiseless_symmetric}.
As in \ref{appendix_subsec:proof_of_thm1},
for any non-zero matrix $\pmb{Z}\in\mathbb{R}^{n_1\times n_2}$ such that $P_{\Omega}(\pmb{Z}) = \pmb{0}$, we want to show the following inequality:
$$
\| \pmb{M} + \pmb{Z} \|_* > \| \pmb{M} \|_*.
$$
For any $\pmb{Y}\in\mathbb{R}^{n_1\times n_2}$ such that $P_{\Omega^c}(\pmb{Y}) = \pmb{0}$, i.e., $\langle \pmb{Y}, \pmb{Z} \rangle = 0$,
\begin{align*}
\| \pmb{M} + \pmb{Z} \|_*
&\geq
\| \pmb{M} \|_* + \langle \pmb{U}\pmb{V}^\top - P_{{T}}(\pmb{Y}), P_{{T}}(\pmb{Z}) \rangle
+ \| P_{{T}^\perp}(\pmb{Z}) \|_* - \langle P_{{T}^\perp}(\pmb{Y}), P_{{T}^\perp}(\pmb{Z}) \rangle
\\&
\geq
\| \pmb{M} \|_* - \| \pmb{U}\pmb{V}^\top - P_{{T}}(\pmb{Y}) \|_F \cdot \| P_{{T}}(\pmb{Z}) \|_F
+ \| P_{{T}^\perp}(\pmb{Z}) \|_* - \| P_{{T}^\perp}(\pmb{Y}) \| \cdot \| P_{{T}^\perp}(\pmb{Z}) \|_*,
\end{align*}
where the Cauchy–Schwarz inequality is used in the last inequality unlike \ref{appendix_subsec:proof_of_thm1}.

\paragraph{1) Bound of $\| P_{{T}}(\pmb{Z}) \|_F$:}
We can derive that 
\begin{align}
\label{eq:B1_bound1}
\| P_{{T}}(\pmb{Z}) \|_F^2
\leq \frac{n_1 n_2}{|\Omega|(1-\theta - \gamma)} \cdot \|P_\Omega P_{{T}} (\pmb{Z})\|_F^2
\leq \frac{n_1 n_2}{|\Omega|(1-\theta - \gamma)} \cdot \|P_\Omega P_{{{T}}^\perp} (\pmb{Z})\|_F^2
\leq \frac{n_1 n_2}{|\Omega|(1-\theta - \gamma)} \cdot \|P_{{{T}}^\perp} (\pmb{Z})\|_*^2
\end{align}
where the first inequality is from Lemmas \ref{lemma:rectangular_bound} and \ref{lemma:frobenius_bound},
the second inequality holds by the fact that $P_{\Omega}(\pmb{Z}) = \pmb{0}$,
and the third inequality holds by the fact that $\|P_\Omega(\pmb{A})\|_F \leq \|\pmb{A}\|_F \leq \|\pmb{A}\|_*$ for any matrix $\pmb{A}$.

\paragraph{2) Bounds of $\| \pmb{U}\pmb{V}^\top - P_{{T}}(\pmb{Y}) \|_F$ and $\| P_{{T}^\perp}(\pmb{Y}) \|$:}
Here, we let $\pmb{Y} = \frac{n_1 n_2}{|\Omega|} P_{\Omega}(\pmb{U} \pmb{V}^\top)$. Then,
\begin{equation}
\| \pmb{U}\pmb{V}^\top - P_{{T}}(\pmb{Y}) \|_F
=
\bigg\| \pmb{U}\pmb{V}^\top - \frac{n_1 n_2}{|\Omega|}  P_{{T}} P_{\Omega}(\pmb{U} \pmb{V}^\top)  \bigg\|_F
\leq
\sqrt{\theta^2 + \gamma^2} \cdot \| \pmb{U} \pmb{V}^\top \|_F
=
\sqrt{r(\theta^2 + \gamma^2)}
\label{eq:B1_bound2}
\end{equation}
where the inequality holds by Lemma \ref{lemma:rectangular_bound}.

Also,
\begin{align}
&\| P_{{T}^\perp}(\pmb{Y}) \|
= 
\| P_{{T}^\perp}( \pmb{Y} - \pmb{U} \pmb{V}^\top) \|
\leq
\| \pmb{Y} - \pmb{U} \pmb{V}^\top \|
\leq
\frac{\sqrt{\mu_0 r n_1 n_2}}{|\Omega|}\cdot ( \xi_{1,\mathcal{G}} + \xi_{2,\mathcal{G}} + \psi_{\mathcal{G}} ) \cdot \sqrt{n_2} \cdot \| \pmb{V} \|_{2,\infty}
\nonumber \\&\leq
\frac{\sqrt{\mu_0 r n_1 n_2}}{|\Omega|}\cdot ( \xi_{1,\mathcal{G}} + \xi_{2,\mathcal{G}} + \psi_{\mathcal{G}} ) \cdot \sqrt{n_2} \cdot \sqrt{\frac{\mu_0 r}{n_2}}
=
\frac{\mu_0 r \sqrt{n_1 n_2}}{|\Omega|} \cdot ( \xi_{1,\mathcal{G}} + \xi_{2,\mathcal{G}} + \psi_{\mathcal{G}} )
=\gamma
\label{eq:B1_bound3}
\end{align}
where the first inequality holds by Lemma \ref{lemma:l2bound_general}, and the second inequality holds by the assumption \textbf{A1}.

Using \eqref{eq:B1_bound1}, \eqref{eq:B1_bound2} and \eqref{eq:B1_bound3}, we derive the following inequality:
\begin{align}
\| \pmb{M} + \pmb{Z} \|_*
&\geq
\| \pmb{M} \|_* 
- \sqrt{r(\theta^2 + \gamma^2)}\cdot 
\sqrt{\frac{n_1 n_2}{|\Omega|(1-\theta - \gamma)}} \cdot \|P_{{{T}}^\perp} (\pmb{Z})\|_*
+
(1-\gamma)\cdot \|P_{{{T}}^\perp} (\pmb{Z})\|_*
\nonumber \\&=
\| \pmb{M} \|_* + \| P_{{T}^\perp}(\pmb{Z}) \|_* \cdot
\bigg\{
1 - \gamma
- \sqrt{ \frac{n_1 n_2 r(\theta^2 + \gamma^2)}{|\Omega|(1-\theta - \gamma)}}
\bigg\}.
\label{eq:B1_bound4}
\end{align}
If $\gamma
+ \sqrt{ \frac{n_1 n_2 r(\theta^2 + \gamma^2)}{|\Omega|(1-\theta - \gamma)}}< 1$,
then the desired result holds.

\subsection{Proof of Theorem \ref{thm:noisy_rectangular}}
\label{appendix_subsec:proof_of_thm4}

Note that
$$
\| \hat{\pmb{M}} - \pmb{M} \|_F^2
=
\| P_{\Omega} (\hat{\pmb{M}} - \pmb{M} ) \|_F^2
+ \| P_{\Omega^c} (\hat{\pmb{M}} - \pmb{M} ) \|_F^2.
$$

\paragraph{1) Bound of $\| P_{\Omega^c} (\hat{\pmb{M}} - \pmb{M} ) \|_F^2$:}
Let $\pmb{Z} = P_{\Omega^c} ( \hat{\pmb{M}} - \pmb{M} )$.
Note that $\|\pmb{Z} \|_F^2 = \| P_{{T}} (\pmb{Z}) \|_F^2 + \| P_{{T}^\perp} (\pmb{Z}) \|_F^2$.
As shown in \eqref{eq:B1_bound1},
$$
\| P_{{T}}(\pmb{Z}) \|_F^2
\leq \frac{n_1 n_2}{|\Omega|(1-\theta - \gamma)} \cdot \|P_{{{T}}^\perp} (\pmb{Z})\|_F^2
$$

Also, as in the proof of Theorem \ref{thm:noisy_symmetric},
we can derive that
$$
\| P_{{T}^\perp}(\pmb{Z}) \|_*
\leq
\| P_{\Omega} (\pmb{M} - \hat{\pmb{M}}) \|_*
\cdot 
\bigg\{
1 - \gamma
- \sqrt{ \frac{n_1 n_2 r(\theta^2 + \gamma^2)}{|\Omega|(1-\theta - \gamma)}}
\bigg\}^{-1}.
$$
Therefore,
\begin{align*}
&\| P_{\Omega^c} (\hat{\pmb{M}} - \pmb{M} ) \|_F^2
= 
\| \pmb{Z} \|_F^2
\\&\leq
\frac{n_1 n_2}{|\Omega|(1-\theta - \gamma)} \cdot \|P_{{{T}}^\perp} (\pmb{Z})\|_F^2 
+ \| P_{{T}^\perp} (\pmb{Z}) \|_F^2
\\&\leq
\bigg(1+\frac{n_1 n_2}{|\Omega|(1-\theta - \gamma)}\bigg)
\cdot
\| P_{{T}^\perp}(\pmb{Z}) \|_*^2
\\&\leq
\bigg(1+\frac{n_1 n_2}{|\Omega|(1-\theta - \gamma)}\bigg)
\cdot 
\bigg\{
1 - \gamma
- \sqrt{ \frac{n_1 n_2 r(\theta^2 + \gamma^2)}{|\Omega|(1-\theta - \gamma)}} 
\bigg\}^{-2}
\cdot
\| P_{\Omega} (\pmb{M} - \hat{\pmb{M}}) \|_*^2
\end{align*}

\paragraph{2) Bound of $\| P_{\Omega} (\hat{\pmb{M}} - \pmb{M} ) \|_F$:}
As in the proof of Theorem \ref{thm:noisy_symmetric},
we have that
$$
\| P_{\Omega} (\hat{\pmb{M}} - \pmb{M} ) \|_F
\leq 2\delta.
$$

Now, we have the upper bound of $\| \hat{\pmb{M}} - \pmb{M} \|_F^2$ as follows:
\begin{align*}
&\| \hat{\pmb{M}} - \pmb{M} \|_F^2
\leq
\| P_{\Omega} (\hat{\pmb{M}} - \pmb{M} ) \|_F^2
+
\bigg(1+\frac{n_1 n_2}{|\Omega|(1-\theta - \gamma)}\bigg)
\cdot 
\bigg\{
1 - \gamma
- \sqrt{ \frac{n_1 n_2 r(\theta^2 + \gamma^2)}{|\Omega|(1-\theta - \gamma)}} 
\bigg\}^{-2}
\cdot
\| P_{\Omega} (\pmb{M} - \hat{\pmb{M}}) \|_*^2
\\&\leq
\bigg[ 1 + 
\min(n_1, n_2)\cdot
\bigg(1+\frac{n_1 n_2}{|\Omega|(1-\theta - \gamma)}\bigg)
\cdot 
\bigg\{
1 - \gamma
- \sqrt{ \frac{n_1 n_2 r(\theta^2 + \gamma^2)}{|\Omega|(1-\theta - \gamma)}} 
\bigg\}^{-2}
\bigg]
\cdot
\| P_{\Omega} (\hat{\pmb{M}} - \pmb{M} ) \|_F^2
\\&\leq
\bigg[ 1 + 
\min(n_1, n_2)\cdot
\bigg(1+\frac{n_1 n_2}{|\Omega|(1-\theta - \gamma)}\bigg)
\cdot 
\bigg\{
1 - \gamma
- \sqrt{ \frac{n_1 n_2 r(\theta^2 + \gamma^2)}{|\Omega|(1-\theta - \gamma)}} 
\bigg\}^{-2}
\bigg]
\cdot
(2\delta)^2.
\end{align*}
If 
$\gamma
+ \sqrt{ \frac{n_1 n_2 r(\theta^2 + \gamma^2)}{|\Omega|(1-\theta - \gamma)}}
< 1-c \leq 1$ 
and $\theta + \gamma < 1-c' \leq 1$
for some positive constants $c$ and $c'$, then we derive the following inequality:
$$
\| \hat{\pmb{M}} - \pmb{M} \|_F
\leq
2\delta + 2\delta \cdot \sqrt{\frac{\min(n_1, n_2)}{c^2}\cdot \{ 1 + (c'p)^{-1}\}}
$$
where $p = \frac{|\Omega|}{n_1 n_2}$.
Then we obtain the desired result by letting $c=\frac{1}{2}$.

\section{Other Lemmas}

\begin{lemma}
\label{lemma:l2bound_general}
Suppose that the assumption \textup{\textbf{A1}} holds. 
For any $\pmb{X} \in \mathbb{R}^{n_2 \times r}$ and $\pmb{Y} \in \mathbb{R}^{n_1 \times r}$,
$$
\bigg\| \frac{n_1 n_2}{|\Omega|} P_\Omega (\pmb{U}\pmb{X}^\top) - \pmb{U}\pmb{X}^\top \bigg\|
\leq 
\frac{\sqrt{\mu_0 r n_1 n_2}}{|\Omega|}\cdot ( \xi_{1,\mathcal{G}} + \xi_{2,\mathcal{G}} + \psi_{\mathcal{G}} ) \cdot \sqrt{n_2} \cdot \| \pmb{X} \|_{2,\infty}
$$ 
and
$$
\bigg\| \frac{n_1 n_2}{|\Omega|} P_\Omega (\pmb{Y}\pmb{V}^\top) - \pmb{Y}\pmb{V}^\top \bigg\|
\leq \frac{\sqrt{\mu_0 r n_1 n_2}}{|\Omega|}\cdot ( \xi_{1,\mathcal{G}} + \xi_{2,\mathcal{G}} + \psi_{\mathcal{G}} ) \cdot \sqrt{n_1} \cdot \| \pmb{Y} \|_{2,\infty}.
$$ 
\end{lemma}

\begin{proof}
We can derive that 
\begin{align*}
&\bigg\| \frac{n_1 n_2}{|\Omega|} P_\Omega (\pmb{U}\pmb{X}^\top) - \pmb{U}\pmb{X}^\top \bigg\|
=
\sup_{\substack{\pmb{x}\in\mathbb{R}^{n_1}, \|\pmb{x}\|=1 \\ \pmb{y}\in\mathbb{R}^{n_2}, \|\pmb{y}\|=1}}
\pmb{x}^\top \bigg\{ \frac{n_1 n_2}{|\Omega|} P_\Omega (\pmb{U}\pmb{X}^\top)
- \pmb{U} \pmb{X}^\top \bigg\} \pmb{y}
\\&= 
\sup_{\substack{\pmb{x}\in\mathbb{R}^{n_1}, \|\pmb{x}\|=1 \\ \pmb{y}\in\mathbb{R}^{n_2}, \|\pmb{y}\|=1}}
\sum_{i=1}^{n_1} \sum_{j=1}^{n_2} \sum_{k=1}^{r} x_i \bigg(\frac{n_1 n_2}{|\Omega|} A_{\mathcal{G},i,j} - 1\bigg) U_{i,k} X_{j,k} y_j
\\&= 
\sup_{\substack{\pmb{x}\in\mathbb{R}^{n_1}, \|\pmb{x}\|=1 \\ \pmb{y}\in\mathbb{R}^{n_2}, \|\pmb{y}\|=1}}
\sum_{k=1}^{r} ( \pmb{x}\circ \pmb{U}_{:,k} )^\top \bigg(\frac{n_1 n_2}{|\Omega|} \pmb{A}_{\mathcal{G}} - \pmb{1}_{n_1}\pmb{1}_{n_2}^\top \bigg) ( \pmb{y}\circ \pmb{X}_{:,k} )
\\&\leq
\sup_{\substack{\pmb{x}\in\mathbb{R}^{n_1}, \|\pmb{x}\|=1 \\ \pmb{y}\in\mathbb{R}^{n_2}, \|\pmb{y}\|=1}}
\sum_{k=1}^{r} 
\frac{n_1 n_2}{|\Omega|} \cdot \| \pmb{x}\circ \pmb{U}_{:,k} \| \cdot \| \pmb{y}\circ \pmb{X}_{:,k} \| \cdot (\xi_{1,\mathcal{G}}+\xi_{2,\mathcal{G}}+\psi_{\mathcal{G}} )
\\&\leq
\frac{n_1 n_2}{|\Omega|} \cdot 
(\xi_{1,\mathcal{G}}+\xi_{2,\mathcal{G}}+\psi_{\mathcal{G}} ) \cdot
\sup_{\substack{\pmb{x}\in\mathbb{R}^{n_1}, \|\pmb{x}\|=1 \\ \pmb{y}\in\mathbb{R}^{n_2}, \|\pmb{y}\|=1}}
\sqrt{\sum_{k=1}^{r} \sum_{i=1}^{n_1} x_i^2 U_{i,k}^2} \cdot
\sqrt{\sum_{k=1}^{r} \sum_{j=1}^{n_2} y_j^2 X_{j,k}^2}
\\&\leq
\frac{n_1 n_2}{|\Omega|} \cdot 
(\xi_{1,\mathcal{G}}+\xi_{2,\mathcal{G}}+\psi_{\mathcal{G}} ) \cdot
\| \pmb{U}\|_{2,\infty} \cdot \| \pmb{X} \|_{2,\infty}
\\&\leq
\frac{\sqrt{\mu_0 r n_1 n_2}}{|\Omega|}\cdot ( \xi_{1,\mathcal{G}} + \xi_{2,\mathcal{G}} + \psi_{\mathcal{G}} ) \cdot \sqrt{n_2} \cdot \| \pmb{X} \|_{2,\infty},
\end{align*}
where the first inequality holds by Lemma \ref{lemma:adjacency_matrix}, the second inequality holds by Cauchy-Schwarz inequality, the third inequality holds by the fact that $\sum_{i}\sum_{j}a_i^2 B_{i,j}^2 \leq (\max_{i} \sum_{j} B_{i,j}^2) \cdot (\sum_{i} a_i^2)$ for any vector $\pmb{a}$ and matrix $\pmb{B}$,
and the last inequality holds by the assumption \textbf{A1}.
We can show the bound of $\big\| \frac{n_1 n_2}{|\Omega|} P_\Omega (\pmb{Y}\pmb{V}^\top) - \pmb{Y}\pmb{V}^\top \big\|$ in a similar way.
\end{proof}

\begin{lemma}
\label{lemma:frobenius_bound}
For any matrix $\pmb{Z} \in \mathbb{R}^{n_1\times n_2}$ and any linear operators $P_1$ and $P_2$ from $\mathbb{R}^{n_1\times n_2}$ to $\mathbb{R}^{n_1\times n_2}$,
if the following inequality holds:
$$
\| c_1 P_1 P_2 P_1 (\pmb{Z}) - P_1 (\pmb{Z}) \|_F \leq c_2 \| P_1 (\pmb{Z}) \|_F
$$
for some positive constants $c_1$ and $c_2$, then
$$
\| P_1 (\pmb{Z}) \|_F \leq \sqrt{c_1 (1-c_2)^{-1}} \| P_2 P_1 (\pmb{Z}) \|_F.
$$
\end{lemma}

\begin{proof}

\begin{align*}
\| P_2 P_1 (\pmb{Z}) \|_F^2
&=
\langle P_2 P_1 (\pmb{Z}), P_2 P_1 (\pmb{Z}) \rangle
=
c_1^{-1} \langle P_1 (\pmb{Z}), c_1 P_1 P_2 P_1 (\pmb{Z}) \rangle
\\&=
c_1^{-1} \| P_1 (\pmb{Z}) \|_F^2 
+ c_1^{-1} \langle P_1 (\pmb{Z}), c_1 P_1 P_2 P_1 (\pmb{Z}) - P_1 (\pmb{Z}) \rangle
\\&\geq
c_1^{-1} \| P_1 (\pmb{Z}) \|_F^2 
- c_1^{-1} \| P_1 (\pmb{Z}) \|_F \cdot \| c_1 P_1 P_2 P_1 (\pmb{Z}) - P_1 (\pmb{Z}) \|_F
\\&\geq
c_1^{-1} \| P_1 (\pmb{Z}) \|_F^2 
- c_1^{-1} \| P_1 (\pmb{Z}) \|_F \cdot c_2 \| P_1 (\pmb{Z}) \|_F
\\&=
c_1^{-1} (1-c_2) \| P_1 (\pmb{Z}) \|_F^2.
\end{align*}

\end{proof}

\medskip

\begin{lemma}
\label{lemma:adjacency_matrix}
For any $\pmb{x} \in \mathbb{R}^{n_1}$ such that $\pmb{x}\perp \pmb{1}_{n_1}$, and for any $\pmb{y} \in \mathbb{R}^{n_2}$,
$$
|\pmb{x}^\top \pmb{A}_{\mathcal{G}} \pmb{y} | \leq \| \pmb{x} \| \cdot \| \pmb{y} \| \cdot (\xi_{1,\mathcal{G}}+\psi_{\mathcal{G}} ).
$$

Similarly, for any $\pmb{x} \in \mathbb{R}^{n_1}$, and for any $\pmb{y} \in \mathbb{R}^{n_2}$ such that $\pmb{y}\perp \pmb{1}_{n_2}$,
$$
|\pmb{x}^\top \pmb{A}_{\mathcal{G}} \pmb{y} | \leq \| \pmb{x} \| \cdot \| \pmb{y} \| \cdot (\xi_{2,\mathcal{G}}+\psi_{\mathcal{G}} ).
$$

Lastly, for any $\pmb{x} \in \mathbb{R}^{n_1}$ and $\pmb{y} \in \mathbb{R}^{n_2}$,
$$
\bigg|~\pmb{x}^\top \bigg(\pmb{1}_{n_1}\pmb{1}_{n_2}^\top - \frac{n_1 n_2}{|\Omega|}\pmb{A}_{\mathcal{G}}\bigg) \pmb{y} ~\bigg| 
\leq \frac{n_1 n_2}{|\Omega|} \cdot \| \pmb{x} \| \cdot \| \pmb{y} \| \cdot (\xi_{1,\mathcal{G}}+\xi_{2,\mathcal{G}}+\psi_{\mathcal{G}} ).
$$

\end{lemma}

\begin{proof}
For any $\pmb{x} \in \mathbb{R}^{n_1}$,
let us write $\pmb{x} = a_x \pmb{1}_{n_1} + b_x \pmb{w}_{x}$ where $\pmb{w}_{x}$ is a unit vector that is perpendicular to $\pmb{1}_{n_1}$, $a_x$ and $b_x$ are some scalars, and $\pmb{w}_{x}$, $a_x$ and $b_x$ are properly determined by $\pmb{x}$.
Note that $|a_x| \leq \frac{1}{\sqrt{n_1}}\|\pmb{x}\|$ and $|b_x| \leq \|\pmb{x}\|$.
Likewise, write $\pmb{y} = a_y \pmb{1}_{n_2} + b_y \pmb{w}_{y}$ for any $\pmb{y} \in \mathbb{R}^{n_2}$.

We will first derive the upper bounds of $|\pmb{w}_x^\top \pmb{A}_{\mathcal{G}} \pmb{1}_{n_2}|$, $|\pmb{1}_{n_1}^\top \pmb{A}_{\mathcal{G}} \pmb{w}_y|$ and $|\pmb{w}_x^\top \pmb{A}_{\mathcal{G}} \pmb{w}_y|$.
Note that
\begin{align*}
\pmb{w}_x^\top \pmb{A}_{\mathcal{G}} \pmb{1}_{n_2}
&= 
\pmb{w}_x^\top (\Delta_{\mathcal{U},1,\mathcal{G}}, \dots, \Delta_{\mathcal{U},n_1,\mathcal{G}})^\top
\\&= 
\pmb{w}_x^\top (\Delta_{\mathcal{U},1,\mathcal{G}}, \dots, \Delta_{\mathcal{U},n_1,\mathcal{G}})^\top
- \pmb{w}_x^\top \pmb{1}_{n_1}\cdot \bigg(\frac{1}{n_1} \sum_{i=1}^{n_1} \Delta_{i,\mathcal{U},\mathcal{G}}\bigg)
\\&=
\pmb{w}_x^\top 
\bigg( \Delta_{\mathcal{U},1,\mathcal{G}} - \frac{1}{n_1} \sum_{i=1}^{n_1} \Delta_{i,\mathcal{U},\mathcal{G}},\dots, \Delta_{\mathcal{U},n_1,\mathcal{G}} - \frac{1}{n_1} \sum_{i=1}^{n_1} \Delta_{i,\mathcal{U},\mathcal{G}}\bigg).
\end{align*}
Hence, $|\pmb{w}_x^\top \pmb{A}_{\mathcal{G}} \pmb{1}_{n_2}| \leq \|\pmb{w}_x\|\cdot \sqrt{\sum_{i=1}^{n_1} \bigg( \Delta_{i,\mathcal{U},\mathcal{G}} - \frac{1}{n_1} \sum_{i=1}^{n_1} \Delta_{i,\mathcal{U},\mathcal{G}} \bigg)^2 }
= \sqrt{n_2} \cdot \xi_{1,\mathcal{G}}$.
Similarly, we have $|\pmb{1}_{n_1}^\top \pmb{A}_{\mathcal{G}} \pmb{w}_y| \leq \sqrt{n_1} \cdot \xi_{2,\mathcal{G}}$. Also,
\begin{align*}
&\pmb{w}_x^\top \pmb{A}_{\mathcal{G}} \pmb{w}_y
=
\frac{1}{2}
\begin{pmatrix}
\pmb{w}_x \\ \pmb{w}_y
\end{pmatrix}^\top
\begin{pmatrix}
\pmb{0} & \pmb{A}_{\mathcal{G}} \\
\pmb{A}_{\mathcal{G}}^\top & \pmb{0}
\end{pmatrix}
\begin{pmatrix}
\pmb{w}_x \\ \pmb{w}_y
\end{pmatrix}
=
\frac{1}{2}
\begin{pmatrix}
\pmb{w}_x \\ \pmb{w}_y
\end{pmatrix}^\top
\begin{pmatrix}
\pmb{D}_{\mathcal{U},\mathcal{G}} & \pmb{0} \\
\pmb{0} & \pmb{D}_{\mathcal{V},\mathcal{G}}
\end{pmatrix}
\begin{pmatrix}
\pmb{w}_x \\ \pmb{w}_y
\end{pmatrix}
-
\frac{1}{2}
\begin{pmatrix}
\pmb{w}_x \\ \pmb{w}_y
\end{pmatrix}^\top
\pmb{L}_{\mathcal{G}}
\begin{pmatrix}
\pmb{w}_x \\ \pmb{w}_y
\end{pmatrix}
\\&\leq
\frac{\max\{\Delta_{i,\mathcal{U},\mathcal{G}}~;~i\in[n_1]\} + \max\{\Delta_{j,\mathcal{V},\mathcal{G}}~;~j\in[n_2]\}}{2}
- \frac{1}{2}\cdot (\|\pmb{w}_x \|^2 + \|\pmb{w}_y\|^2) \cdot \varphi_{\mathcal{G}}
= 
\Delta_{\max, \mathcal{G}} - \varphi_{\mathcal{G}}
\end{align*}
where $\pmb{D}_{\mathcal{U},\mathcal{G}}$ and $\pmb{D}_{\mathcal{V},\mathcal{G}}$ are diagonal matrices whose diagonal elements are the degrees of the nodes in $\mathcal{U}$ and $\mathcal{V}$ for graph $\mathcal{G}$, respectively, and
$\pmb{L}_{\mathcal{G}}$ is the Laplacian matrix corresponding to $\begin{psmallmatrix}
\pmb{0} & \pmb{A}_{\mathcal{G}} \\
\pmb{A}_{\mathcal{G}}^\top & \pmb{0}
\end{psmallmatrix}$. Similarly, we can derive the lower bound as follows:
\begin{align*}
&\pmb{w}_x^\top \pmb{A}_{\mathcal{G}} \pmb{w}_y
= 
\pmb{w}_x^\top (\pmb{1}_{n_1}\pmb{1}_{n_2}^\top - \pmb{A}_{\bar{\mathcal{G}}} ) \pmb{w}_y
=
- \pmb{w}_x^\top \pmb{A}_{\bar{\mathcal{G}}} \pmb{w}_y
\geq
-\Delta_{\max, \bar{\mathcal{G}}} + \varphi_{\bar{\mathcal{G}}}.
\end{align*}
Therefore, $|\pmb{w}_x^\top \pmb{A}_{\mathcal{G}} \pmb{w}_y| \leq \max\{\Delta_{\max, \mathcal{G}} - \varphi_{\mathcal{G}}, \Delta_{\max, \bar{\mathcal{G}}} - \varphi_{\bar{\mathcal{G}}} \} = \psi_{\mathcal{G}}$.

Now, consider $\pmb{x} \in \mathbb{R}^{n_1}$ such that $\pmb{x}\perp \pmb{1}_{n_1}$, i.e., $a_x = 0$. For any $\pmb{y} \in \mathbb{R}^{n_2}$,
\begin{align*}
|\pmb{x}^\top \pmb{A}_{\mathcal{G}} \pmb{y}|
&=
|b_x \pmb{w}_x^\top \pmb{A}_{\mathcal{G}} (a_y \pmb{1}_{n_2} + b_y \pmb{w}_y)|
\leq
|a_y| |b_x| |\pmb{w}_x^\top \pmb{A}_{\mathcal{G}} \pmb{1}_{n_2}| + |b_x| |b_y| |\pmb{w}_x^\top \pmb{A}_{\mathcal{G}} \pmb{w}_y|
\\&\leq
\|\pmb{x}\| \cdot \frac{1}{\sqrt{n_2}} \|\pmb{y}\| \cdot \sqrt{n_2} \cdot \xi_{1,\mathcal{G}} + \|\pmb{x}\| \cdot \|\pmb{y}\| \cdot \psi_{\mathcal{G}}
= \|\pmb{x}\| \cdot \|\pmb{y}\| \cdot (\xi_{1,\mathcal{G}}+\psi_{\mathcal{G}} ).
\end{align*}
Likewise, for any $\pmb{x} \in \mathbb{R}^{n_1}$, and for any $\pmb{y} \in \mathbb{R}^{n_2}$ such that $\pmb{y}\perp \pmb{1}_{n_2}$, we can derive that
$$
|\pmb{x}^\top \pmb{A}_{\mathcal{G}} \pmb{y} | \leq \| \pmb{x} \| \cdot \| \pmb{y} \| \cdot (\xi_{2,\mathcal{G}}+\psi_{\mathcal{G}} ).
$$
Lastly, for any $\pmb{x} \in \mathbb{R}^{n_1}$ and $\pmb{y} \in \mathbb{R}^{n_2}$, we can derive that
\begin{align*}
&\bigg|~\pmb{x}^\top \bigg(\pmb{1}_{n_1}\pmb{1}_{n_2}^\top - \frac{n_1 n_2}{|\Omega|}\pmb{A}_{\mathcal{G}}\bigg) \pmb{y} ~\bigg| 
= 
\bigg|~(a_x \pmb{1}_{n_1} + b_x \pmb{w}_x )^\top \bigg(\pmb{1}_{n_1}\pmb{1}_{n_2}^\top - \frac{n_1 n_2}{|\Omega|}\pmb{A}_{\mathcal{G}}\bigg) (a_y \pmb{1}_{n_2} + b_y \pmb{w}_y ) ~\bigg| 
\\&= 
\bigg|~
a_x a_y n_1 n_2 - \frac{n_1 n_2}{|\Omega|}
(a_x a_y |\Omega| 
+ a_x b_y \pmb{1}_{n_1}^\top \pmb{A}_{\mathcal{G}} \pmb{w}_y
+ a_y b_x \pmb{w}_x^\top \pmb{A}_{\mathcal{G}} \pmb{1}_{n_2}
+ b_x b_y \pmb{w}_x^\top \pmb{A}_{\mathcal{G}} \pmb{w}_y )
~\bigg|
\\&\leq
\frac{n_1 n_2}{|\Omega|}
( |a_x| |b_y| |\pmb{1}_{n_1}^\top \pmb{A}_{\mathcal{G}} \pmb{w}_y|
+ |a_y| |b_x| |\pmb{w}_x^\top \pmb{A}_{\mathcal{G}} \pmb{1}_{n_2}|
+ |b_x| |b_y| |\pmb{w}_x^\top \pmb{A}_{\mathcal{G}} \pmb{w}_y| )
\\&\leq
\frac{n_1 n_2}{|\Omega|} \cdot \| \pmb{x} \| \cdot \| \pmb{y} \| \cdot (\xi_{1,\mathcal{G}}+\xi_{2,\mathcal{G}}+\psi_{\mathcal{G}} ).
\end{align*}

\end{proof}

\section{Additional Simulation Results}
\label{appendix_sec:experimental_results}

For rectangular matrices, we create synthetic data as follows.
We first generate the left and right singular matrices $\pmb{U}\in \mathbb{R}^{500\times r}$ and $\pmb{V}\in \mathbb{R}^{250\times r}$ using standard normal distribution.
We then generate the rank-$r$ matrix $\pmb{M} \in \mathbb{R}^{500\times 250}$ using $\pmb{M} = \pmb{U}\pmb{V}^\top$.
In the scenario of noisy matrices, we randomly generate the entry-wise noise from a normal distribution with mean $0$ and standard deviation $\sigma$.
We try different values of rank $r \in \{10, 20, 30\}$ and noise parameter $\sigma \in \{10^{-4},10^{-5},10^{-6}\}$ in the experiments.

To generate observation graphs with various values of graph properties, we first divide the nodes of each vertex set of $\mathcal{U}$ and $\mathcal{V}$ into two clusters (let $\mathcal{U} = \mathcal{U}_1 \cup \mathcal{U}_2$ and $\mathcal{V} = \mathcal{V}_1 \cup \mathcal{V}_2$)
and sample edges in $(\mathcal{U}_1 \otimes \mathcal{V}_1)\cup(\mathcal{U}_2 \otimes \mathcal{V}_2)$ 
and
$(\mathcal{U}_1 \otimes \mathcal{V}_2)\cup(\mathcal{U}_2 \otimes \mathcal{V}_1)$
with probability of $p \in (0,1)$ and $q \in (0,1)$, respectively.
For each $p+q \in \{0.2, 0.4, \dots, 1.2\}$, we try different values of $p$ and $q$ so that the graphs have diverse values of $\xi_{1,\mathcal{G}} + \xi_{2,\mathcal{G}} + \psi_{\mathcal{G}}$.
Specifically, we have $\xi_{1,\mathcal{G}} + \xi_{2,\mathcal{G}} + \psi_{\mathcal{G}}$ fall within one of the ranges $120$ to $130$, $130$ to $140$, $\dots$, or $190$ to $200$.

We use an Augmented Lagrangian Method \citep{lin2010augmented} to solve the constrained nuclear norm minimization problems \eqref{eq:nnm} and \eqref{eq:cnnm}.
When solving \eqref{eq:cnnm} for approximate matrix completion, 
we set the tuning parameter $\delta$ to be $4\sigma\sqrt{|\Omega|}$.
For evaluation, we calculate the relative error $\frac{\|\pmb{M} - \hat{\pmb{M}} \|_F}{\| \pmb{M} \|_F}$ in each experiment.
In exact matrix completion, we consider a trial to be successful if the relative error is less than $0.01$, and compute the success ratio over $30$ trials with different random seeds.
In approximate matrix completion, we calculate the average relative error over $30$ trials.

Figure \ref{fig:rectangular_noiseless_different_rank} shows the result of exact matrix completion in noiseless matrix case.
We can observe that as $\xi_{1,\mathcal{G}} + \xi_{2,\mathcal{G}} + \psi_{\mathcal{G}}$ or $r$ increases, or $p+q$ decreases (i.e., the number of observed entries decreases), the success ratio decreases, which supports Theorem \ref{thm:noiseless_rectangular}.
Figure \ref{fig:rectangular_noisy_different_rank_sigma} demonstrates the result of approximate matrix completion in noisy matrix case.
In the three plots above, we can observe that as $\xi_{1,\mathcal{G}} + \xi_{2,\mathcal{G}} + \psi_{\mathcal{G}}$ or $r$ increases, or $p+q$ decreases, the average of relative errors tends to increase.
In the three plots below, we can see that as the noise parameter $\sigma$ decreases, the average of relative errors decreases.
These observations are consistent with our finding in Theorem \ref{thm:noisy_rectangular}.

\begin{figure}[t]
	\centering
	\includegraphics[width=1\textwidth]{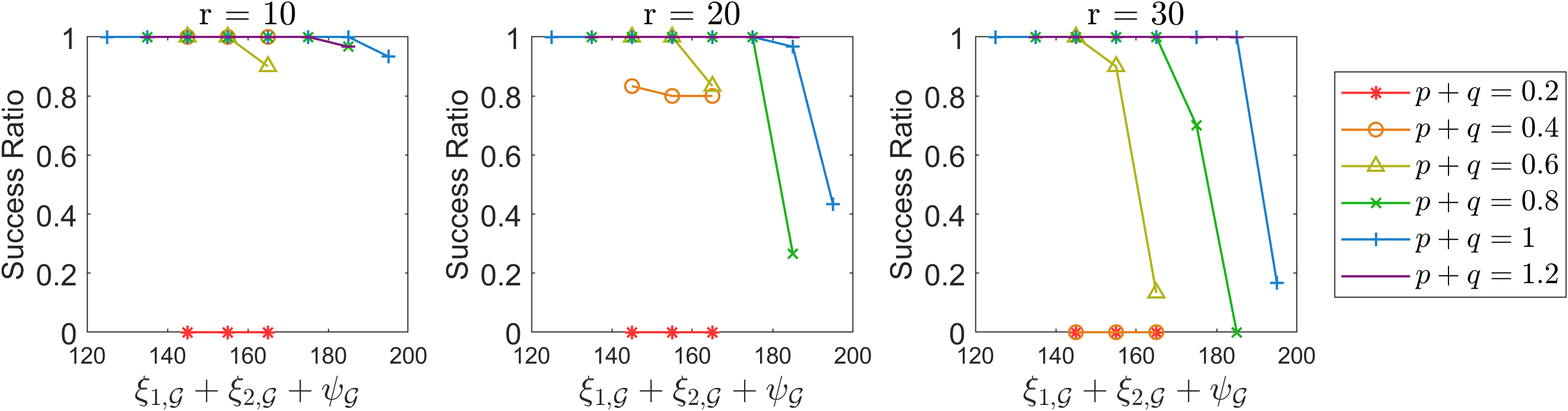}
	\caption{Success ratio of exact matrix completion versus graph property $\xi_{1,\mathcal{G}} + \xi_{2,\mathcal{G}} + \psi_{\mathcal{G}}$ with different rank $r$ for rectangular matrices. Different line colors or markers indicate different values of observation probability $p+q$.}
	\label{fig:rectangular_noiseless_different_rank}
\end{figure}

\begin{figure}[t]
	\centering
	\includegraphics[width=1\textwidth]{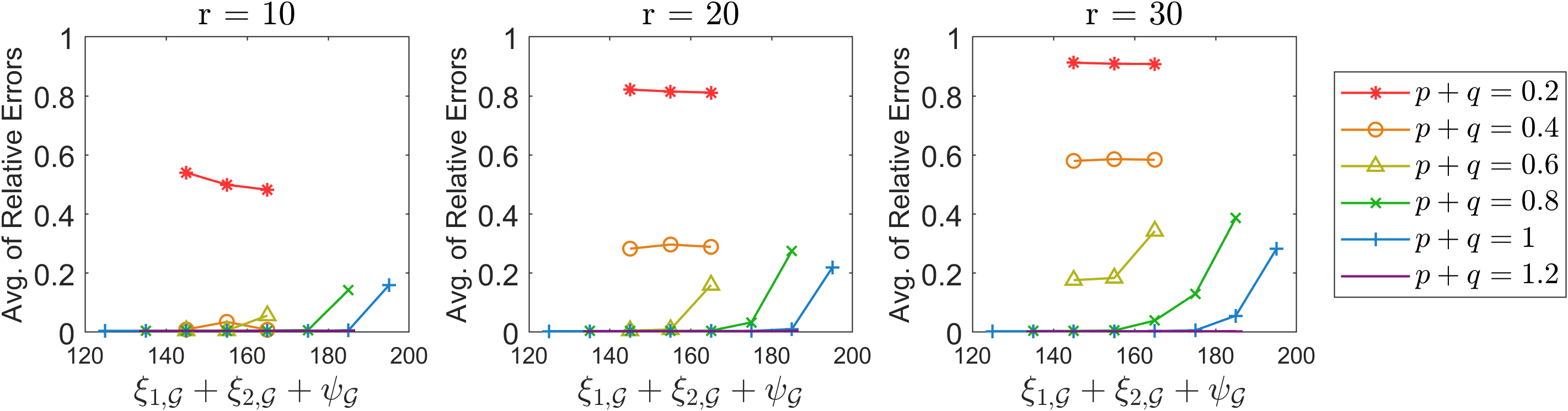}
	\\[1em]
	\includegraphics[width=1\textwidth]{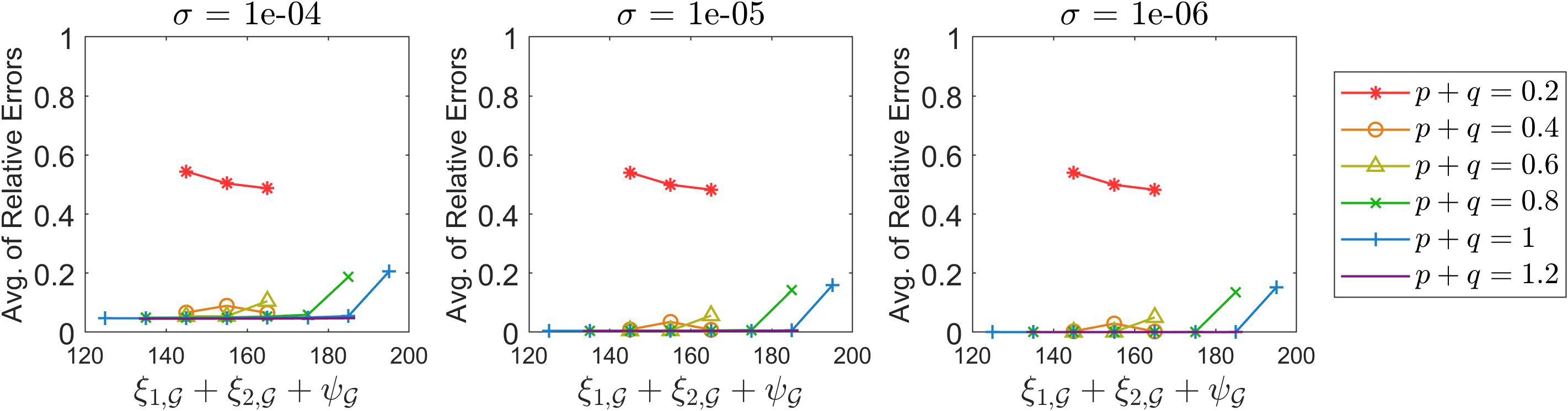}
	\caption{Average of relative errors in approximate matrix completion versus graph property $\xi_{1,\mathcal{G}} + \xi_{2,\mathcal{G}} + \psi_{\mathcal{G}}$ for rectangular matrices.
	Three plots above show results for different rank $r$ with fixed noise parameter $\sigma=10^{-5}$.
	Three plots below are of different noise parameter $\sigma$ with fixed rank $r=10$.
Different line colors or markers indicate different values of observation probability $p+q$.}
	\label{fig:rectangular_noisy_different_rank_sigma}
\end{figure}

Lastly, we want to verify whether the performance of the algorithm is solely determined by the factors derived in our theorems. 
Here, we utilize the rescaled parameter
$\frac{|\Omega|}{\mu_0 r (2\xi_{\mathcal{G}} + \psi_{\mathcal{G}})}$ or $\frac{|\Omega|}{\mu_0 r (\xi_{1,\mathcal{G}} + \xi_{2,\mathcal{G}} + \psi_{\mathcal{G}})}$.
In Figure \ref{fig:overlap}, we can observe that the curves share almost the same pattern across different settings of rank $r$. 
This empirical finding provides justification for our theorems.

\begin{figure}[t]
	\centering
	\includegraphics[width=0.3\textwidth]{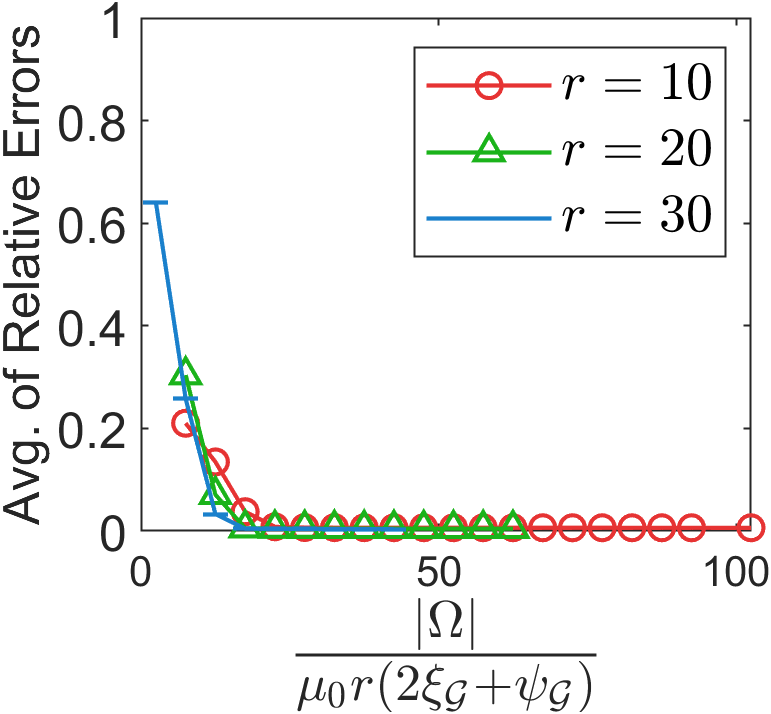}
	\includegraphics[width=0.3\textwidth]{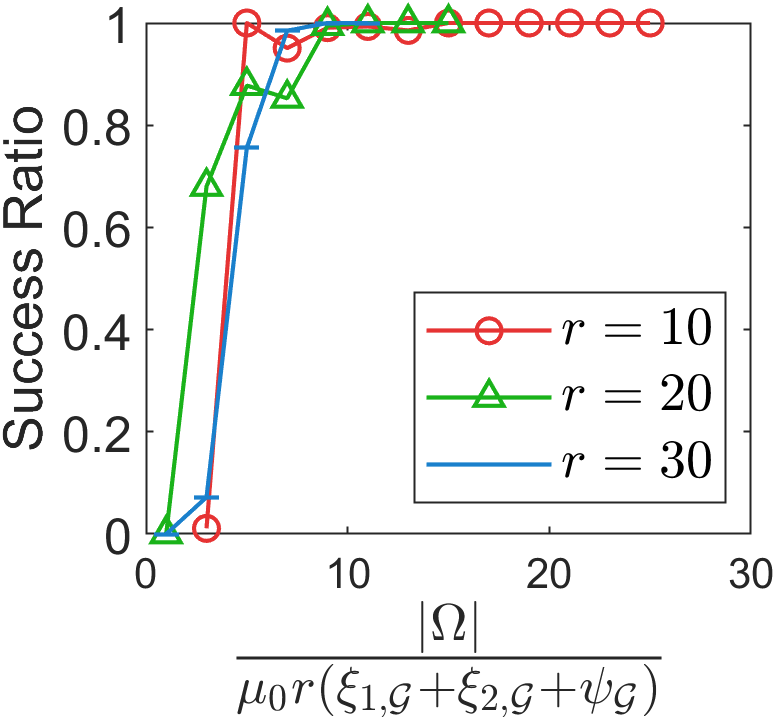}
	\includegraphics[width=0.3\textwidth]{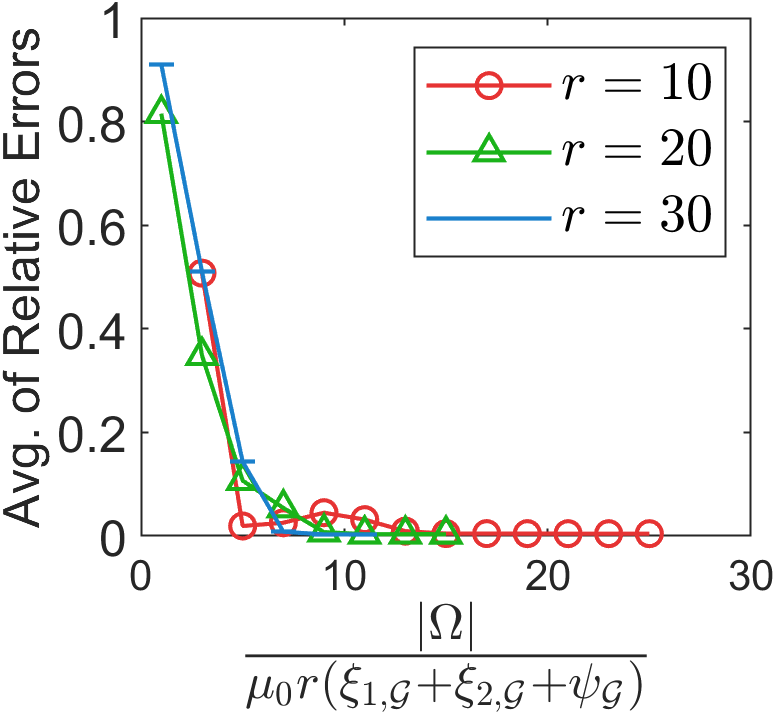}
	\caption{Success ratio or average of relative errors versus rescaled parameter for different rank $r$.
	From left to right: cases of noisy symmetric matrices, noiseless rectangular matrices, and noisy rectangular matrices.}
	\label{fig:overlap}
\end{figure}

\end{document}